\documentclass[a4paper,fleqn]{cas-dc}
\usepackage{booktabs}
\usepackage{bm}
\usepackage{amssymb}
\usepackage{multirow}
\usepackage{subcaption}
\usepackage{soul}
\usepackage{xcolor}
\usepackage{mathtools}
\usepackage{amsthm}

\theoremstyle{plain}
\newtheorem{theorem}{Theorem}

\theoremstyle{definition}

\theoremstyle{remark}

\usepackage[authoryear,longnamesfirst]{natbib}

\def\tsc#1{\csdef{#1}{\textsc{\lowercase{#1}}\xspace}}
\tsc{WGM}
\tsc{QE}

\begin{document}
\let\WriteBookmarks\relax
\def\floatpagepagefraction{1}
\def\textpagefraction{.001}

\shorttitle{Effective rank and forward compatibility in class incremental learning}

\shortauthors{J.~Kim, W.~Lee, M.~Eo~and~W.~Rhee}  

\title [mode = title]{Improving Forward Compatibility in Class Incremental Learning by Increasing Representation Rank and Feature Richness}



%

\author[1]{Jaeill~Kim}[orcid=0000-0002-1088-2882]
\cormark[1]
\ead{jaeill.kim@linecorp.com}
\author[2]{Wonseok~Lee}[orcid=0009-0004-8995-339X]
\ead{dnjstjr1017@snu.ac.kr}
\author[3]{Moonjung~Eo}[orcid=0000-0002-0114-8010]
\cormark[1]
\ead{moonj@lgresearch.ai}
\author[2,4,5]{Wonjong~Rhee}[orcid=0000-0002-2590-8774]
\cormark[2]
\ead{wrhee@snu.ac.kr}






\affiliation[1]{organization={LINE Investment Technologies},
            addressline={117 Bundangnaegok-ro, Bundang-gu}, 
            city={Seongnam-si},
            postcode={13529}, 
            country={South Korea}}
\affiliation[2]{organization={Interdisciplinary Program in Artificial Intelligence, Seoul National University},
            addressline={1 Gwanak-ro, Gwanak-gu}, 
            city={Seoul},
            postcode={08826}, 
            country={South Korea}}
\affiliation[3]{organization={LG AI Research},
            addressline={30 Magokjungang-10ro, Gangseo-gu}, 
            city={Seoul},
            postcode={07796}, 
            country={South Korea}}
\affiliation[4]{organization={Department of Intelligence and Information, Seoul National University},
            addressline={1 Gwanak-ro, Gwanak-gu}, 
            city={Seoul},
            postcode={08826}, 
            country={South Korea}}
\affiliation[5]{organization={Research Institute for Convergence Science, Seoul National University},
            addressline={1 Gwanak-ro, Gwanak-gu}, 
            city={Seoul},
            postcode={08826}, 
            country={South Korea}}






\cortext[1]{Work performed while at Seoul National University.}
\cortext[2]{Corresponding author}



\begin{abstract}
Class Incremental Learning (CIL) constitutes a pivotal subfield within continual learning, aimed at enabling models to progressively learn new classification tasks while retaining knowledge obtained from prior tasks. Although previous studies have predominantly focused on backward compatible approaches to mitigate catastrophic forgetting, recent investigations have introduced forward compatible methods to enhance performance on novel tasks and complement existing backward compatible methods.
In this study, we introduce effective-Rank based Feature Richness enhancement (RFR) method that is designed for improving forward compatibility. Specifically, this method increases the effective rank of representations during the base session, thereby facilitating the incorporation of more informative features pertinent to unseen novel tasks. Consequently, RFR achieves dual objectives in backward and forward compatibility: minimizing feature extractor modifications and enhancing novel task performance, respectively.
To validate the efficacy of our approach, we establish a theoretical connection between effective rank and the Shannon entropy of representations. Subsequently, we conduct comprehensive experiments by integrating RFR into eleven well-known CIL methods. Our results demonstrate the effectiveness of our approach in enhancing novel-task performance while mitigating catastrophic forgetting. Furthermore, our method notably improves the average incremental accuracy across all eleven cases examined.
\end{abstract}




\maketitle

\section{Introduction}
\label{sec:introduction}

Continual learning, the process of continually acquiring and integrating new knowledge, has emerged as a significant challenge in the field of machine learning. In contrast to conventional learning paradigms that focus on static datasets and fixed tasks, continual learning encompasses the dynamic and ever-changing nature of real-world applications. To address this challenge, class incremental learning (CIL)~\citep{rebuffi2017icarl, hou2019learning, douillard2020podnet, shi2022mimicking}, a subfield of continual learning, primarily focuses on developing techniques that enable adaptive learning to accommodate new classes, while minimizing the detrimental impact of knowledge degradation on previously learned classes.

In CIL, training takes place through multiple incremental sessions, each focusing on a distinct subset of classes that do not overlap with the subsets utilized in other sessions. The \textit{base session} involves training a model from scratch to perform a \textit{base task}, typically involving a large number of classes. Subsequently, in each \textit{novel session}, the model is expected to incrementally learn a new \textit{novel task}, typically involving a smaller number of classes, while retaining its performance on previously learned tasks. The model's performance is evaluated using the classes of all the previous and current tasks, without access to the task identification information.

The primary objective in CIL is to mitigate the detrimental impact of catastrophic forgetting, which refers to a significant decline in the model's performance on previous task classes after learning new classes from a subsequent task.
To address this challenge, most of the previous works have focused on addressing the forgetting problem in the updated model (i.e., backward compatible approach)~\citep{zhou2022forward,shi2022mimicking}.
This is commonly achieved by enforcing similarity between the updated model and its predecessor.
For instance, weight regularization methods enforce similarity in weights~\citep{kirkpatrick2017overcoming,zenke2017continual,aljundi2018memory,chaudhry2018riemannian}, while knowledge distillation methods enforce similarity in representations~\citep{rebuffi2017icarl,li2017learning,castro2018end,hou2019learning,wu2019large}.

\begin{figure*}[t!]
    \centering
    \begin{subfigure}{0.30\linewidth}
        \includegraphics[width=\linewidth]{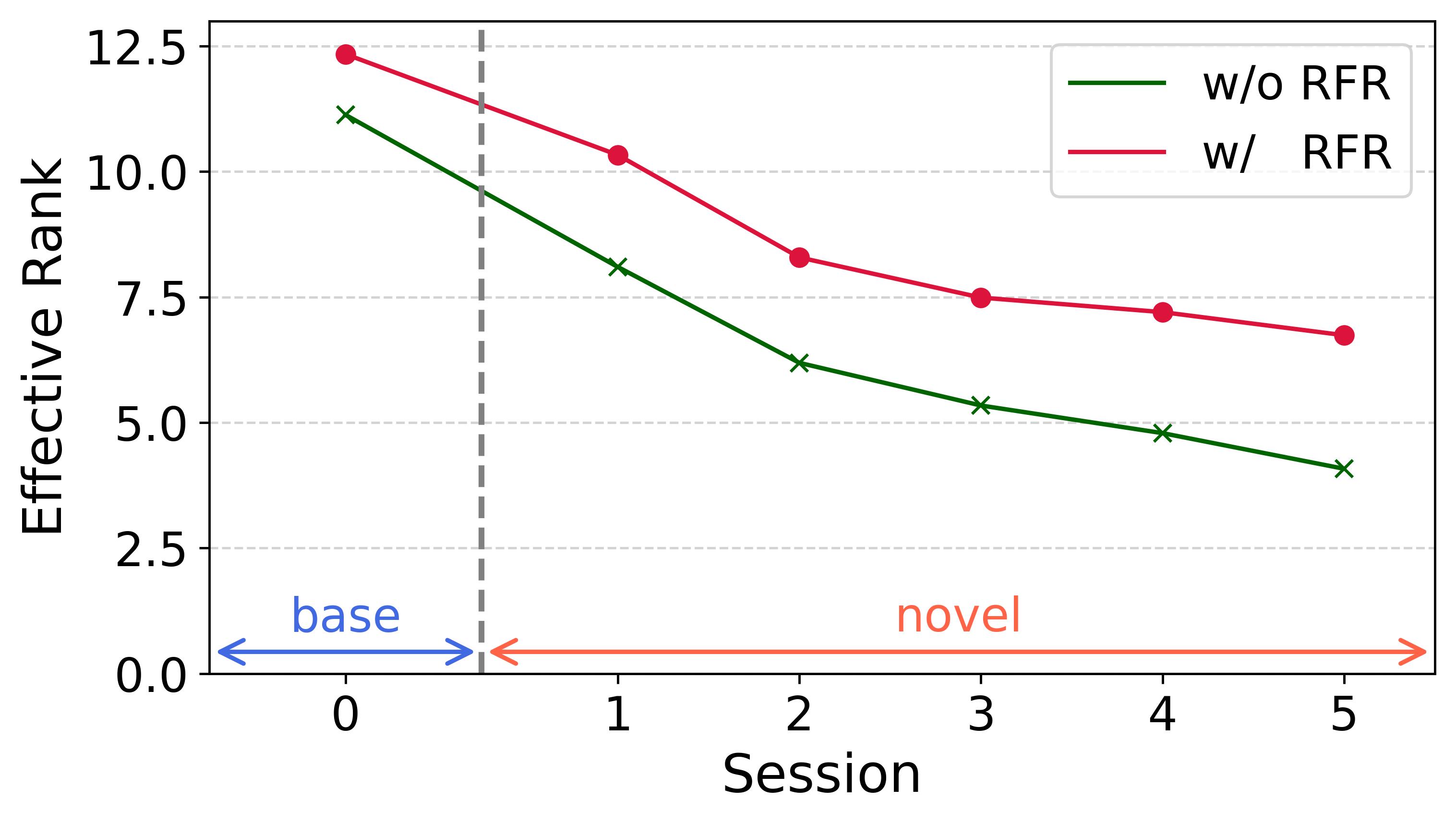}
        \caption{Representation rank}
    \end{subfigure}
    \begin{subfigure}{0.30\linewidth}
        \includegraphics[width=\linewidth]{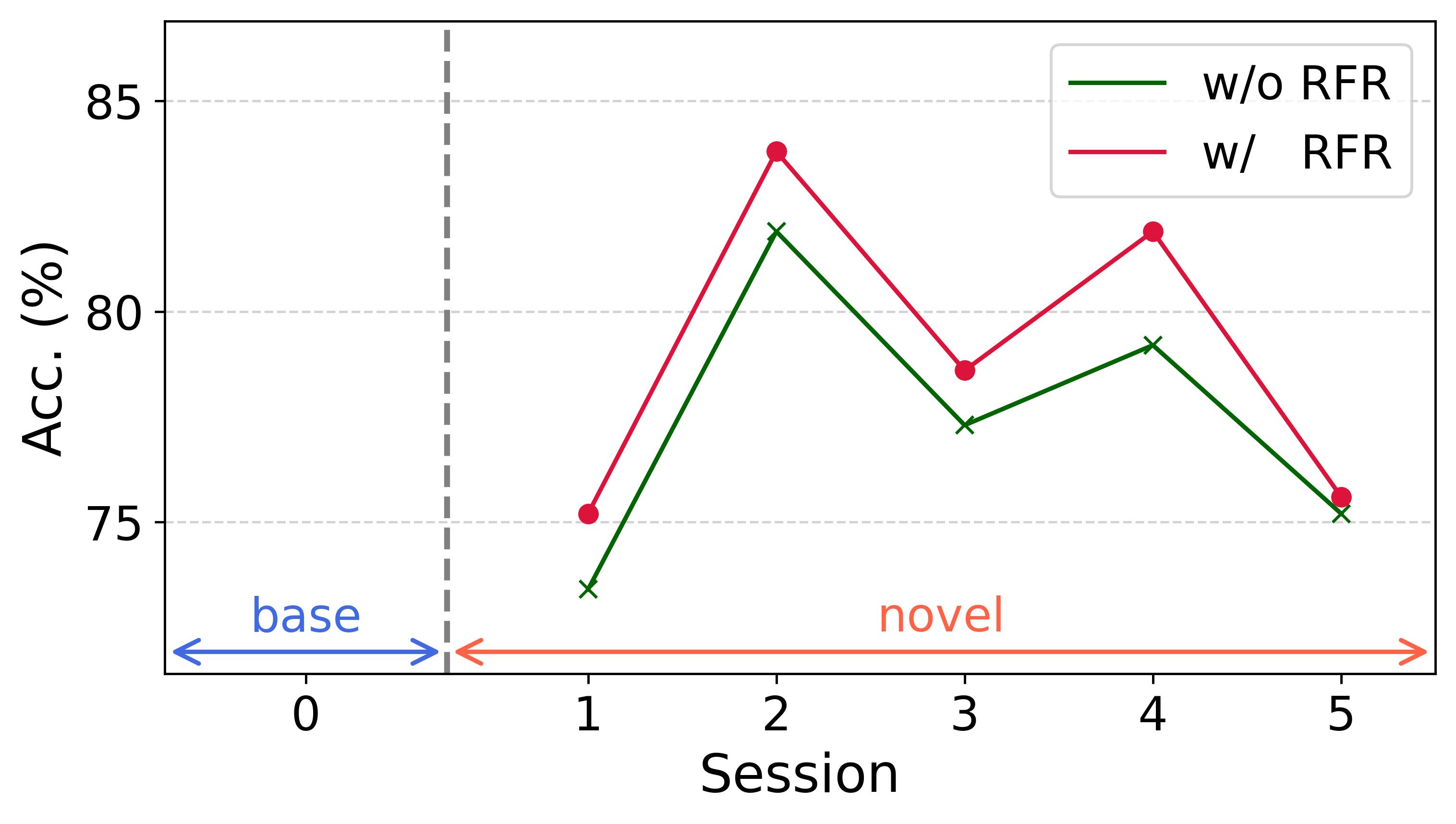}
        \caption{Novel task performance}
    \end{subfigure}
    \begin{subfigure}{0.30\linewidth}
        \includegraphics[width=\linewidth]{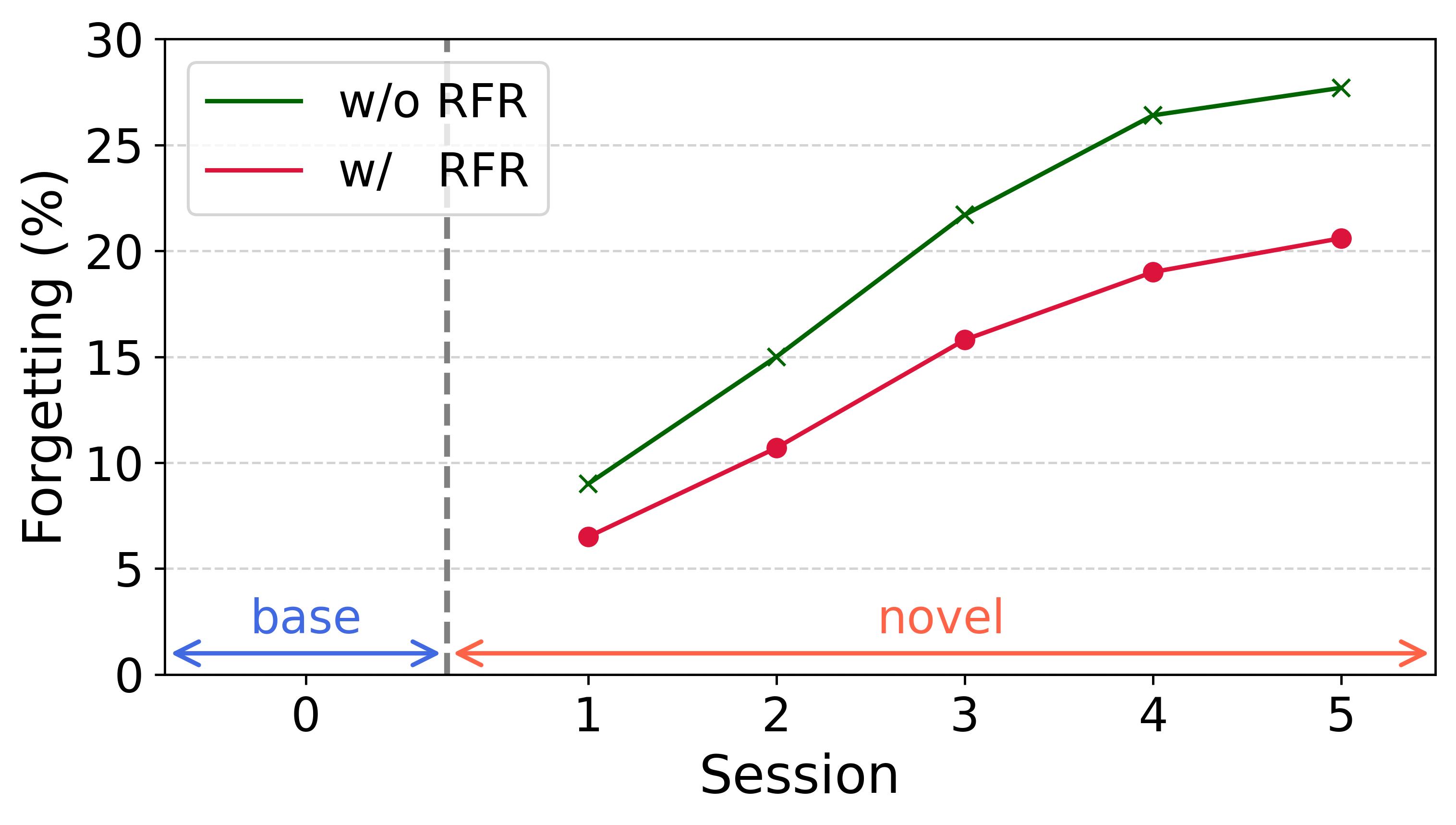}
        \caption{Catastrophic forgetting}
    \end{subfigure}
    \caption{Impact of increasing representation rank during the base session.
    We conduct an analysis of UCIR models with and without the integration of our method. ResNet-18 model is trained for the CIFAR-100 dataset, utilizing 50 base classes and a split size of 10 classes for each novel session.
    (a) Effective rank of the feature extractor.
    (b) Novel task performance in each novel session.
    (c) The degree of catastrophic forgetting that occurs for the base task.
    }
    \label{fig:motivation}
\end{figure*}

In contrast, the concept of forward compatible approach, which aims to facilitate the training of the subsequent tasks, has received relatively little attention.
Recently, a few studies have introduced methods aimed at learning forward-compatible representations, predominantly by leveraging class information.
These methods include employing dual augmentation techniques (class augmentation and semantic augmentation)~\citep{zhu2021class} and enforcing class-wise decorrelations (CwD)~\citep{shi2022mimicking}.
In few-shot class incremental learning (FSCIL), another method called FACT also implements forward compatible method by assigning virtual prototypes and relying on class information~\citep{zhou2022forward}.

To enhance the forward compatibility of the representation,
this study focuses on representation rank and feature richness. To be precise, we utilize \textit{effective rank}~\citep{roy2007effective} in lieu of algebraic rank. Effective rank is a continuous-value extension of algebraic rank. In contrast to algebraic rank, it possesses two advantageous properties: differentiability and the ability to effectively manage extremely small singular values.
Representation rank is a general property of representations, and it is not constrained to specific types of information such as class information.
Therefore, we conjecture that the representation rank can serve as a crucial indicator of the quantity of encoded features in the representation, with higher-rank representations expected to contain rich features that can be beneficial for subsequent tasks.
In this work, we substantiate the conjecture with a theorem and empirical investigations. 
We prove that the Shannon entropy of the representation is maximized when the effective rank is maximized. Because entropy is a quantitative measure of information, larger entropy can be interpreted as richer features. 
Empirically, we show that effective rank increases as more classes are included in a plain supervised learning and we also show that effective rank increases as unsupervised learning proceeds.

To this end, we propose an effective-Rank based Feature Richness enhancement~(RFR) method that increases the effective rank of representation during the base session in order to preserve informative features.
Specifically, we highlight the importance of regularizing the feature extractor during the base session, as it offers two main advantages in the process of learning novel sessions.
First, the performance of novel tasks can be enhanced by utilizing the rich features encoded by the base task.
Second, catastrophic forgetting can be mitigated because novel tasks can leverage the rich features encoded by the base task, leading to minimal modifications to the feature extractor during the learning of novel tasks.
In contrast to IL2A~\citep{zhu2021class}, which explicitly implements both backward and forward compatible methods across base and novel sessions, RFR achieves two distinct methodological objectives solely through a forward compatible approach, without explicit adoption of backward compatible regularization methods during novel sessions.

We have performed extensive experiments to confirm the effectiveness of RFR for class incremental learning. While we defer the explanation of the full results until Section~\ref{sec:experiments}, a glimpse of the experimental results is provided in Figure~\ref{fig:motivation}. Our method can effectively increase the representation rank as shown in Figure~\ref{fig:motivation}(a), can improve the performance of novel tasks as exhibited in Figure~\ref{fig:motivation}(b), and can substantially mitigate catastrophic forgetting as demonstrated in Figure~\ref{fig:motivation}(c). These empirical results provide compelling evidence for the effectiveness of our approach in enhancing the feature richness in the representation.

The main contributions can be summarized as follows: (1) While previous methods focus solely on either forward compatibility or backward compatibility, this study proposes the novel method RFR, which enhances both forward and backward compatibility simultaneously. (2) This study provides both theoretical and empirical evidence supporting the superiority of the method.

\section{Related works}

\subsection{Backward compatible approaches}
Weight regularization methods aim to minimize the weight distance between the feature extractor learned in the previous session and the feature extractor learned in the subsequent session. In this approach, previous works have primarily focused on calculating the importance of each weight to penalize changes of individual weights.
To calculate the importance, several methods have been proposed.
EWC~\citep{kirkpatrick2017overcoming} proposed a diagonal approximation of the Fisher information matrix.
SI~\citep{zenke2017continual} and MAS~\citep{aljundi2018memory} proposed a path integral approach, which accumulates the changes in weights throughout the entire learning trajectory.
RWalk~\citep{chaudhry2018riemannian} combined the Fisher information matrix approach with the path integral approach.

Representation regularization methods aim to prevent forgetting by imposing a penalty on changes in representations. Typically, a regularization is applied during novel sessions, wherein Knowledge Distillation~\citep{hinton2015distilling} plays a key role. The previous session's network acts as the teacher, imparting its knowledge to the student network being trained in the novel session.
iCaRL~\citep{rebuffi2017icarl} employs sigmoid output for knowledge distillation, while other methods~\citep{li2017learning,castro2018end,wu2019large,zhao2020maintaining} utilize temperature-scaled softmax outputs.
UCIR~\citep{hou2019learning} incorporates cosine normalization, less-forget constraint, and inter-class separation to mitigate the adverse effects of the imbalance between previous and new classes.
PODNet~\citep{douillard2020podnet} effectively reduces the difference in pooled intermediate features along the height and width directions through knowledge distillation.

\subsection{Forward compatible approaches}

In the pursuit of establishing forward compatibility within class incremental learning scenarios,
IL2A~\citep{zhu2021class} introduced two distinct augmentation strategies aimed at improving both backward and forward compatibility. The semantic augmentation generates features from stored distribution of previous tasks, and these features are fed to another classification loss for backward compatibility. Simultaneously, the class augmentation method utilizes an augmented class as an auxiliary class to learn transferable and diverse representation, thereby promoting forward compatibility.
The Classwise Decorrelation (CwD)~\citep{shi2022mimicking} method was developed to mimic the behavior of an oracle during the base session. Through empirical investigations, it was discerned that resembling the representation distribution patterns akin to those exhibited by the oracle, characterized by a uniform dispersion of eigenvalues across each class, holds the potential to enhance forward compatibility. To realize this, classwise Frobenius norm of representations was strategically employed during the base session, serving as a mechanism to enforce the desired distribution consistency.

If we consider a broader research area beyond class incremental learning, ForwArd Compatible Training (FACT) also introduced a forward compatible approach for few-shot class incremental learning~\citep{zhou2022forward}.
Within the FACT framework, the concept of virtual classes was incorporated during the training of the base session, effectively allocating embedding space to accommodate upcoming classes. FACT integrates pseudo labels and virtual instances to facilitate effective network training generated through the manifold mixup technique. These elements collectively enhanced the network's adaptability, underscoring FACT's significance in fostering the seamless integration of new classes while upholding established knowledge.

While these forward-compatible approaches have demonstrated their effectiveness, they are all reliant on class information.
Considering that the training data for the base task can encompass considerably many informative features that can be utilized for novel sessions beyond mere class information,
regularization of representations in an unsupervised manner could prove advantageous in encoding richer, forward-compatible features.

\section{Enhancing feature richness by increasing representation rank}
\label{sec:methods}

In this section, we present the details of the proposed method.
The primary objective of our method is to enhance the feature richness in the representation by increasing the rank of the feature extractor's output representation during the base session.

\subsection{Effective rank}

For a set of $N$ samples in a mini-batch, each having an $L_2$-normalized representation vector $\bm{h}_i \in \mathbb{R}^d$ satisfying $||\bm{h}_i||_2=1$ and $N>d$,
the rank of representation matrix $\bm{H} = [\bm{h}_1, \bm{h}_2, ..., \bm{h}_N]^T \in \mathbb{R}^{N \times d}$ can be quantified as
\begin{align}
\label{eq:plain_rank}
\mathsf{rank}(\bm{H})
= \mathsf{rank}(\bm{U}\bm{\Sigma}\bm{V}^T)
= \mathsf{rank}(\bm{\Sigma}) = \sum_{i=1}^{d} \bm{1}_\mathrm{0<\sigma_i},
\end{align}
where $\bm{U}\bm{\Sigma}\bm{V}^T$ is a singular value decomposition of $\bm{H}$ and $\{\sigma_i\}$ are the singular values arranged in a descending order.

The definition of algebraic rank in Eq.~(\ref{eq:plain_rank}) exhibits two practical problems. The first problem is that it equally counts all positive singular values regardless of their strength. Therefore, it can be misleading when extremely small $\sigma_i$ values exist. For instance, rank is known to be susceptible to noise~\citep{choi2017selecting}. A commonly adopted remedy for this problem is to set a threshold for counting. We call this \textit{thresholded rank} as $\mathsf{trank}$, and it is defined as 
\begin{align}
\label{eq:trank}
\mathsf{trank}(\bm{H},\rho) \triangleq \underset{k}{\arg\min} \left(\rho \cdot \sum_{i=1}^{d}\sigma_i^2  \le \sum_{i=1}^{k}\sigma_i^2   \right),
\end{align}
where $\rho$ is a threshold parameter chosen between $0$ and $1$. Eq.~(\ref{eq:trank}) quantifies the count of the largest singular values that collectively encompass $\rho$ proportion of the total singular value energy. 
The $\mathsf{trank}$ provides a straightforward way to avoid the first practical problem, but it is still vulnerable to the second problem. 
The second problem is the non-differentiable nature of algebraic rank and $\mathsf{trank}$. They are both integer-valued and thus not differentiable. A direct consequence is the difficulty for integrating rank or $\mathsf{trank}$ into an end-to-end learning. An elegant work-around for this problem is known as \textit{effective rank}~\citep{roy2007effective}. Effective rank is defined as 
\begin{align}
\label{eq:erank}
\mathsf{erank}(\bm{H}) \triangleq \exp \left(-\sum_{i=1}^{d} \lambda_{i} \log{\lambda_{i}}\right),
\end{align}
where $\lambda_{i} \triangleq \sigma_i^2/N$ and $\{\lambda_{i}\}$ corresponds to the set of eigenvalues for $\bm{H}^T\bm{H}/N$.
We note that $\sum_{i=1}^{d} \lambda_{i} = 1$ because $\sum_{i=1}^{d} \lambda_{i} = \sum_{i=1}^{d}\sigma_i^2/N=tr(\bm{H}^T\bm{H}/N)=tr(\sum_{i=1}^N \bm{h}_i\bm{h}_i^T/N)\\=\sum_{i=1}^N tr(\bm{h}_i\bm{h}_i^T)/N=1$.
While being continuous, the effective rank is known to satisfy a list of properties~\citep{roy2007effective}. 
In particular, the following lower bound can be derived. 
\begin{align}
\label{eq:bound_rank}
\mathsf{erank}(\bm{H}) \le \mathsf{rank}(\bm{H}).
\end{align}
The logarithm of $\mathsf{erank}$ turns out to be the same as the von Neumann entropy~\citep{nielsen2002quantum, wilde2013quantum} and its use for learning representation has been extensively studied in~\citep{kim2023vne}. In our work, we adopt $\mathsf{erank}$ as the main concept and the logarithm of $\mathsf{erank}$ as the main method for controlling feature richness.

\subsection{Empirical investigation of rank vs. feature richness}

To understand the relation between the three types of rank and feature richness in representation, we have devised two experiments. For the supervised learning experiment, we have controlled the number of classes used for the training and measured the rank values. For the unsupervised learning experiment, we have performed contrastive learning with SimCLR loss~\citep{chen2020simple} and measured the rank values as the learning epoch increases. The results are shown in Figure~\ref{fig:rank_sup_ssl}. 
In both experiments, algebraic rank is fixed at the maximum value ($d=512$) because it counts even extremely small singular values. Both $\mathsf{trank}$ and $\mathsf{erank}$, however, increase as more classes are used for plain supervised learning or as training epoch increases for unsupervised learning.
For a more comprehensive exploration of the correlation between representation rank and downstream performance in contrastive learning, refer to the analysis provided in~\citep{garrido2023rankme}.
As we will show in Section~\ref{sec:experiments}, a strong relationship between $\mathsf{erank}$ and feature richness also holds for class incremental learning.

\begin{figure}[ht!]
    \centering
    \begin{subfigure}{0.9\linewidth}
        \includegraphics[width=\linewidth]{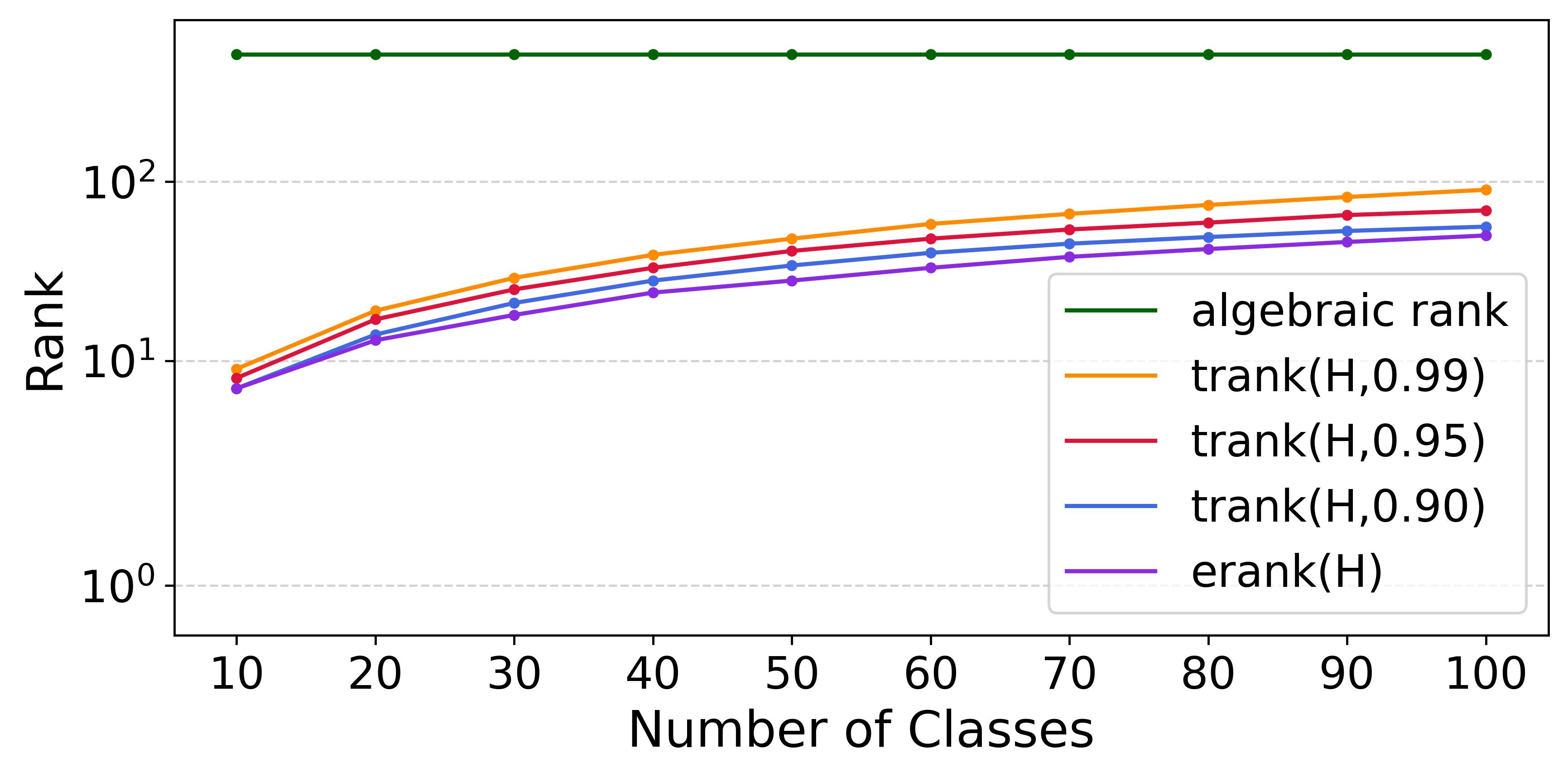}
        \caption{Supervised}
    \end{subfigure}
    \begin{subfigure}{0.9\linewidth}
        \includegraphics[width=\linewidth]{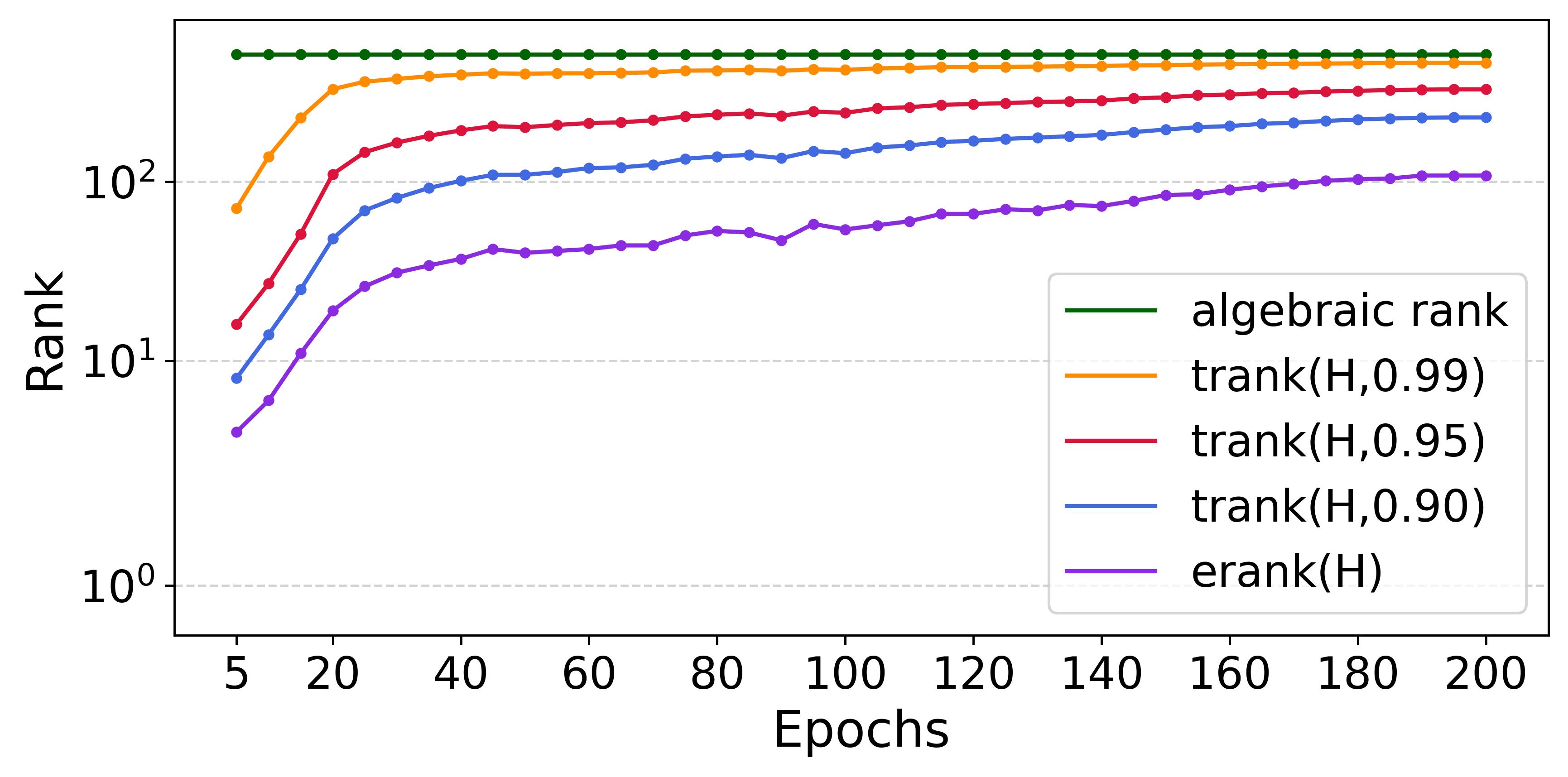}
        \caption{Unsupervised}
    \end{subfigure}
    \caption{Rank vs. feature richness. ResNet-18 was trained using ImageNet-100 dataset.  
(a) Supervised learning --- as we include more classes in the training that starts from scratch, both $\mathsf{trank}$ and $\mathsf{erank}$ increase. Algebraic rank remains at the maximum value. 
(b) Unsupervised learning with SimCLR loss --- as the unsupervised representation learning proceeds, both $\mathsf{trank}$ and $\mathsf{erank}$ increase. Algebraic rank remains at the maximum value.
}\label{fig:rank_sup_ssl}
\end{figure}

\subsection{Proposed method}

In class incremental learning, the training process is divided into the initial base session and the following novel sessions. 
Our goal is to enhance feature richness during the base session by increasing representation rank. 
We adopt $\mathsf{erank}$ as the starting point because it is differentiable. Then, we make an adjustment where we apply a logarithm because of the implementational effectiveness demonstrated in~\citep{kim2023vne}.
The effective-Rank based Feature Richness enhancement~(RFR) method is implemented by including the RFR loss, $\mathcal{L}_\mathit{RFR}$, during the base session as below.
\begin{align}
\mathcal{L} &= \mathcal{L}_\mathit{CrossEntropy} + \alpha \cdot \mathcal{L}_\mathit{RFR} \\
&= \mathcal{L}_\mathit{CrossEntropy} + \alpha \cdot \sum_{i=1}^{d} \lambda_{i} \log{\lambda_{i}},
\label{eq:loss}
\end{align}
where $\alpha>0$ is the strength hyper-parameter. The RFR loss is not applied in the novel sessions, and this decision is analyzed in Section~\ref{subsec:discuss_novel_sessions}.

\subsection{Connection to Shannon entropy}

Shannon entropy is a fundamental measure of information and it quantifies the information contained in a random variable~\citep{cover1999elements}. Therefore, entropy of representation can serve as a measure of feature richness. Assuming Gaussian distribution, we prove a theoretical connection between the proposed RFR and Shannon entropy.

\begin{theorem}
For representation $\bm{h} \in \mathbb{R}^{d}$ that follows a multivariate Gaussian distribution, the entropy of representation is maximized if the effective rank of the representation is maximized.
\end{theorem}
\begin{proof}
Without loss of generality, we consider normalized representation vectors (i.e., $||\bm{h}||_2=1$) that follow a zero-mean multivariate Gaussian distribution (i.e., $\bm{h} \sim \mathcal{N} (0, \Sigma)$), thereby satisfying the condition \(tr(\Sigma)=1\).

The proof is based on two parts. In the first part, we prove that the solution for maximizing the representation entropy is $\Sigma = \frac{1}{d}I$. In the second part, we prove that the solution for maximizing effective rank is also $\Sigma = \frac{1}{d}I$. 

The entropy of multivariate Gaussian distribution, $\mathcal{N} (\mu, \Sigma)$, can be computed as~\citep{cover1999elements}
\begin{align}
\label{eq:thm}
\frac{d}{2}\log{}2\pi e +\frac{1}{2}\log{\det(\Sigma)}, 
\end{align}
where $\det(\Sigma)$ is the determinant of $\Sigma$. By examining the equation, it can be confirmed that entropy is maximized when $\det(\Sigma)$ is maximized under the constraint of \(tr(\Sigma)=1\). The solution for this optimization problem is $\frac{1}{d}I$ because of Hadamard's inequality and inequality of arithmetic and geometric means. First, the following Hadamard's inequality states that the determinant of a positive definite matrix is less than the product of its diagonal elements. 
\begin{align}
\label{eq:Hadamard}
\det(\Sigma) \le \prod_i \Sigma_{ii} \text{, with equality iff } \Sigma \text{ is diagonal.} 
\end{align}
Therefore, all the off-diagonal terms need to be zero to maximize the entropy. Second, $\lambda_i$ needs to be equal to $1/d$ for all $i$ -- otherwise, the inequality of arithmetic and geometric implies that $\prod_i \Sigma_{ii}$ can be increased further while satisfying the sum constraint of \(tr(\Sigma)=1\). 

The proof for $\Sigma = \frac{1}{d}I$ being the solution for maximizing the effective rank in Eq.~(\ref{eq:erank}) is trivial. The logarithm of effective rank is  $-\sum_{i=1}^{d}\lambda_i\log{\lambda_i}$ where 
$\sum_{i=1}^{d} \lambda_{i} = 1$. Because this can be interpreted as the entropy of a probability distribution denoted by $\{\lambda_i\}$, it is maximized by the uniform distribution, i.e., $\lambda_{i}=\frac{1}{d}$.
\end{proof}

It is noteworthy that the Gaussian assumption on the representation~\citep{kingma2013auto,yang2021free} has not only been empirically observed by numerous researchers but also has been theoretically justified in works including~\citep{williams1997computing,neal2012bayesian,lee2017deep,yang2019wide}.

\section{Experiments}
\label{sec:experiments}

We provide an overview of our experimental settings in Section~\ref{subsec:settings}.
Subsequently, we demonstrate the efficacy of our method for enhancing novel task performance in Section~\ref{subsec:forward_compatibility} and for mitigating catastrophic forgetting in Section~\ref{subsec:forgetting}.
In Section~\ref{subsec:sota_experiment}, we demonstrate that our method can improve performance of eleven well-known backward-compatible methods.
In Section~\ref{subsec:further_analysis}, further analysis is provided. In Section~\ref{subsec:ablation_study}, an ablation study is presented.
In Section~\ref{subsec:forward_compatibility} and Section~\ref{subsec:forgetting}, we focus our analysis on the UCIR method, given its prevalent adoption across recent state-of-the-art methodologies. Code is available at: \url{https://github.com/Wonseok1017/NN-RFR}

\begin{figure*}[htbp]
    \centering
    \begin{subfigure}[b]{0.30\linewidth}
        \includegraphics[width=\linewidth]{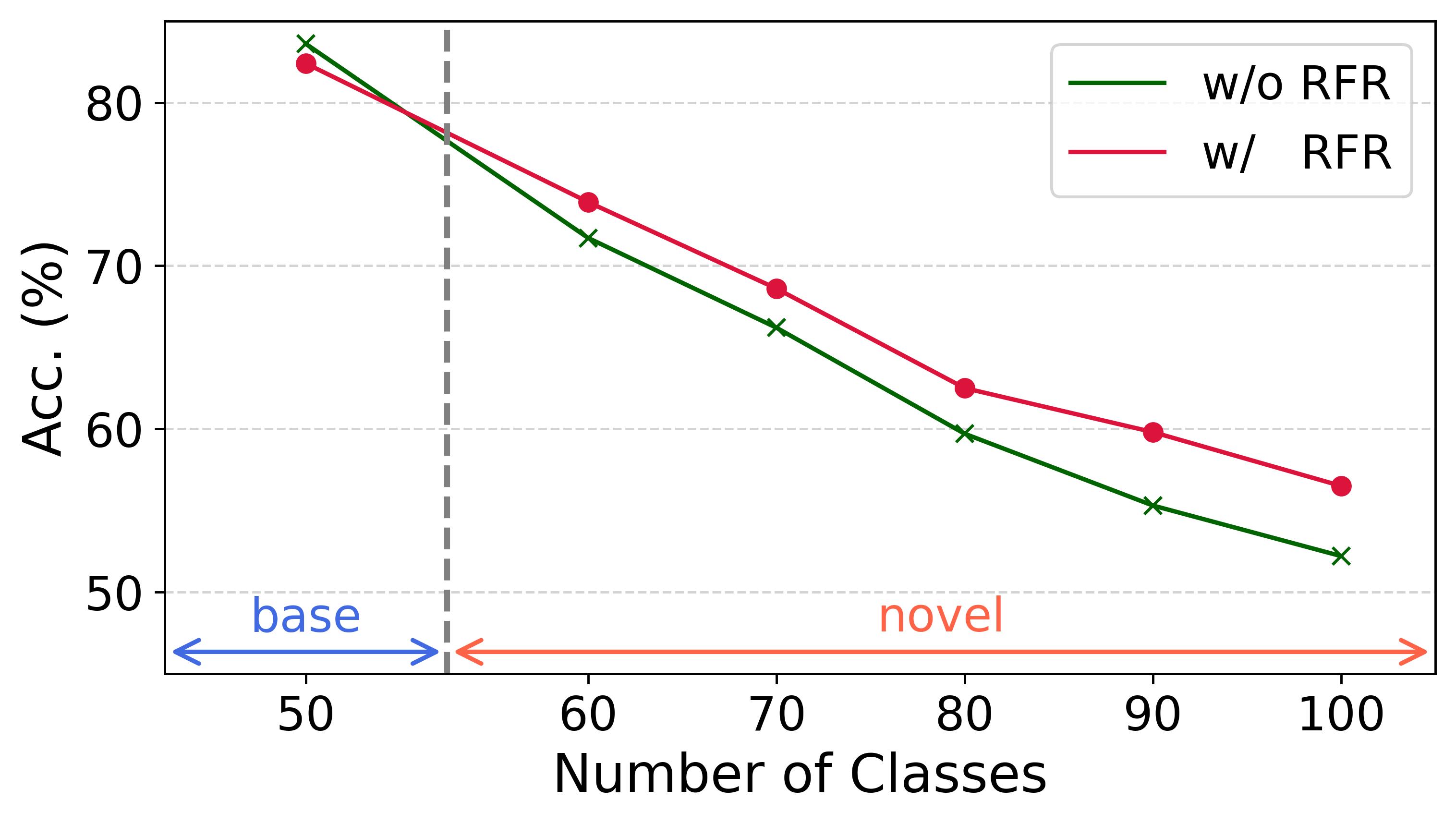}
        \caption{S=10}
    \end{subfigure}
    \begin{subfigure}[b]{0.30\linewidth}
        \includegraphics[width=\linewidth]{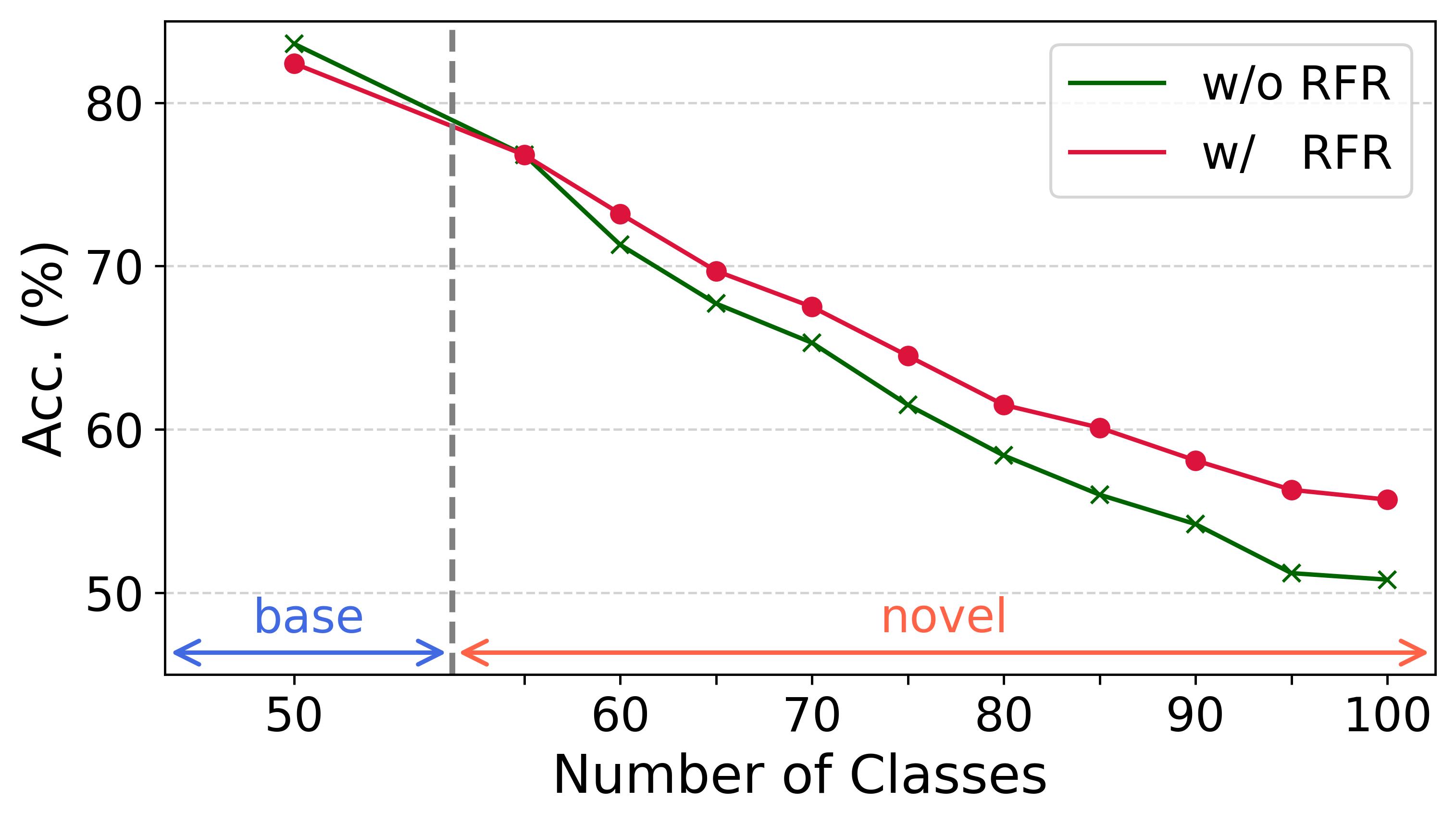}
        \caption{S=5}
    \end{subfigure}
    \begin{subfigure}[b]{0.30\linewidth}
        \includegraphics[width=\linewidth]{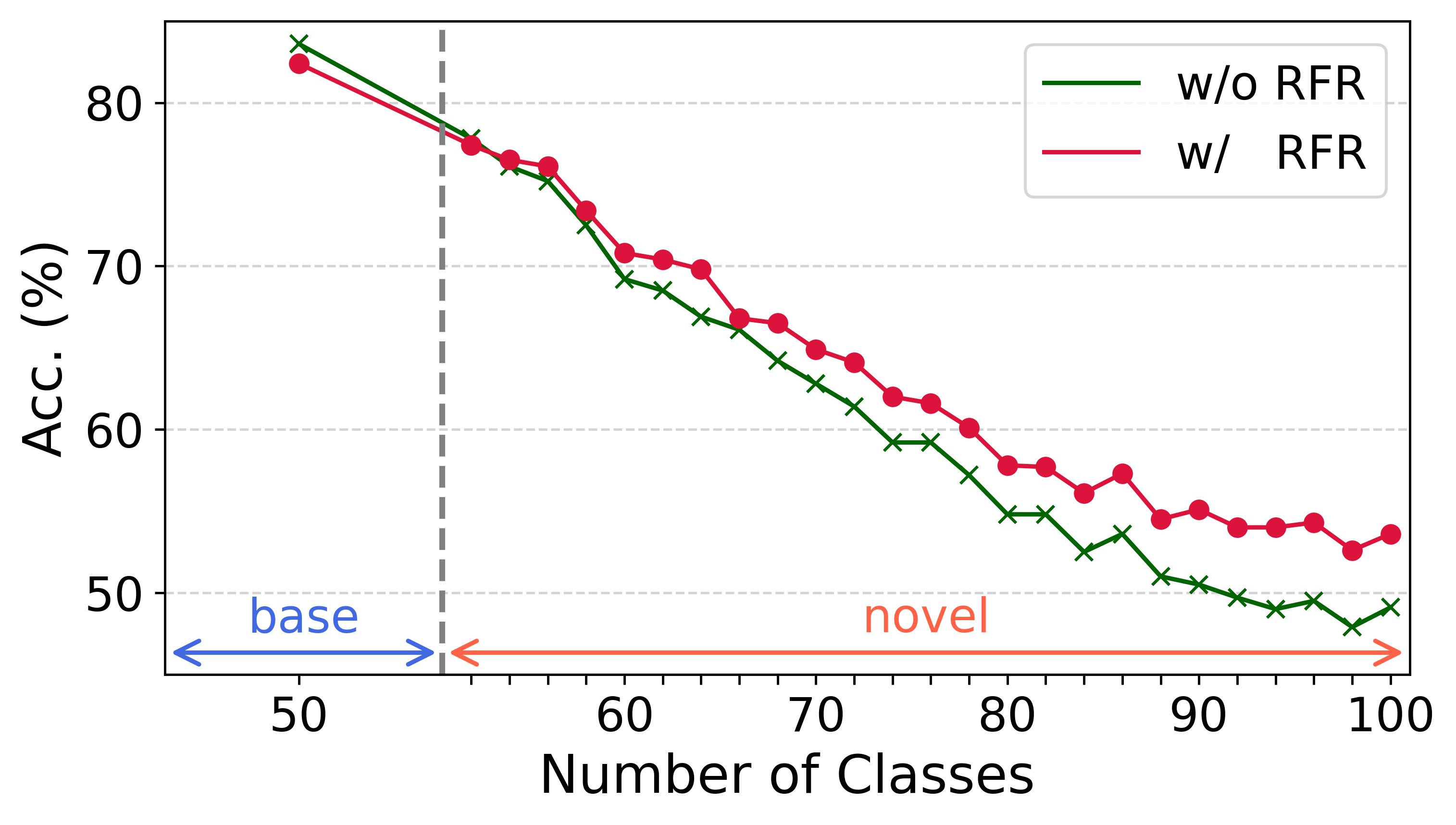}
        \caption{S=2}
    \end{subfigure}
    \caption{Improvements in forward compatibility -- \textit{overall} accuracy at each session is shown for UCIR. 
    Two ResNet-18 models are trained with and without RFR for ImageNet-100 dataset, utilizing 50 base classes under different split sizes (a) 10, (b) 5, and (c) 2 for each novel session.
    The feature extractors trained by the 50 classes of the base task remain \textbf{frozen} during novel sessions.
    }
    \label{fig:acc_freeze_each_phase}
    \centering
    \begin{subfigure}[b]{0.30\linewidth}
        \includegraphics[width=\linewidth]{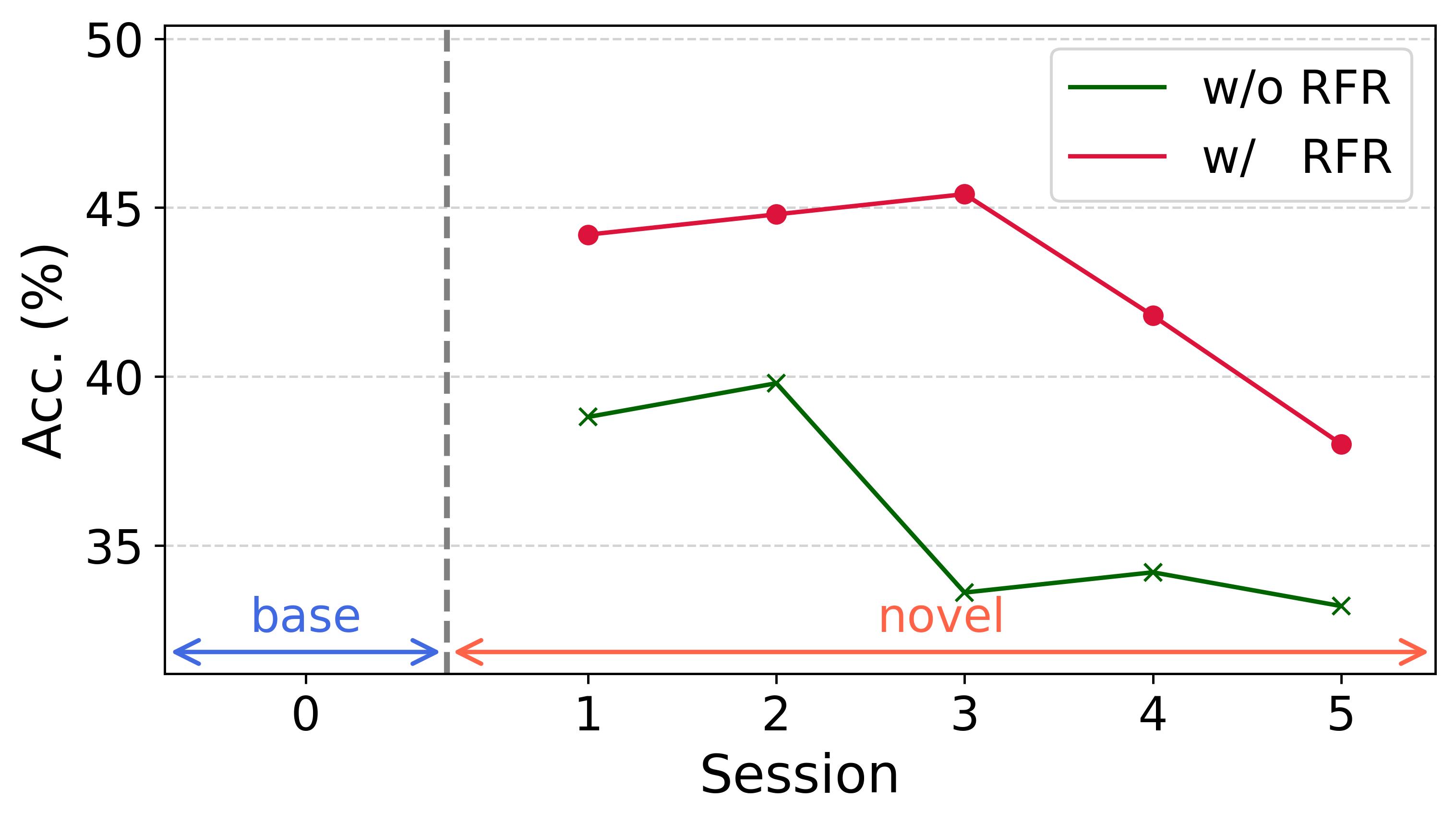}
        \caption{S=10}
    \end{subfigure}
    \begin{subfigure}[b]{0.30\linewidth}
        \includegraphics[width=\linewidth]{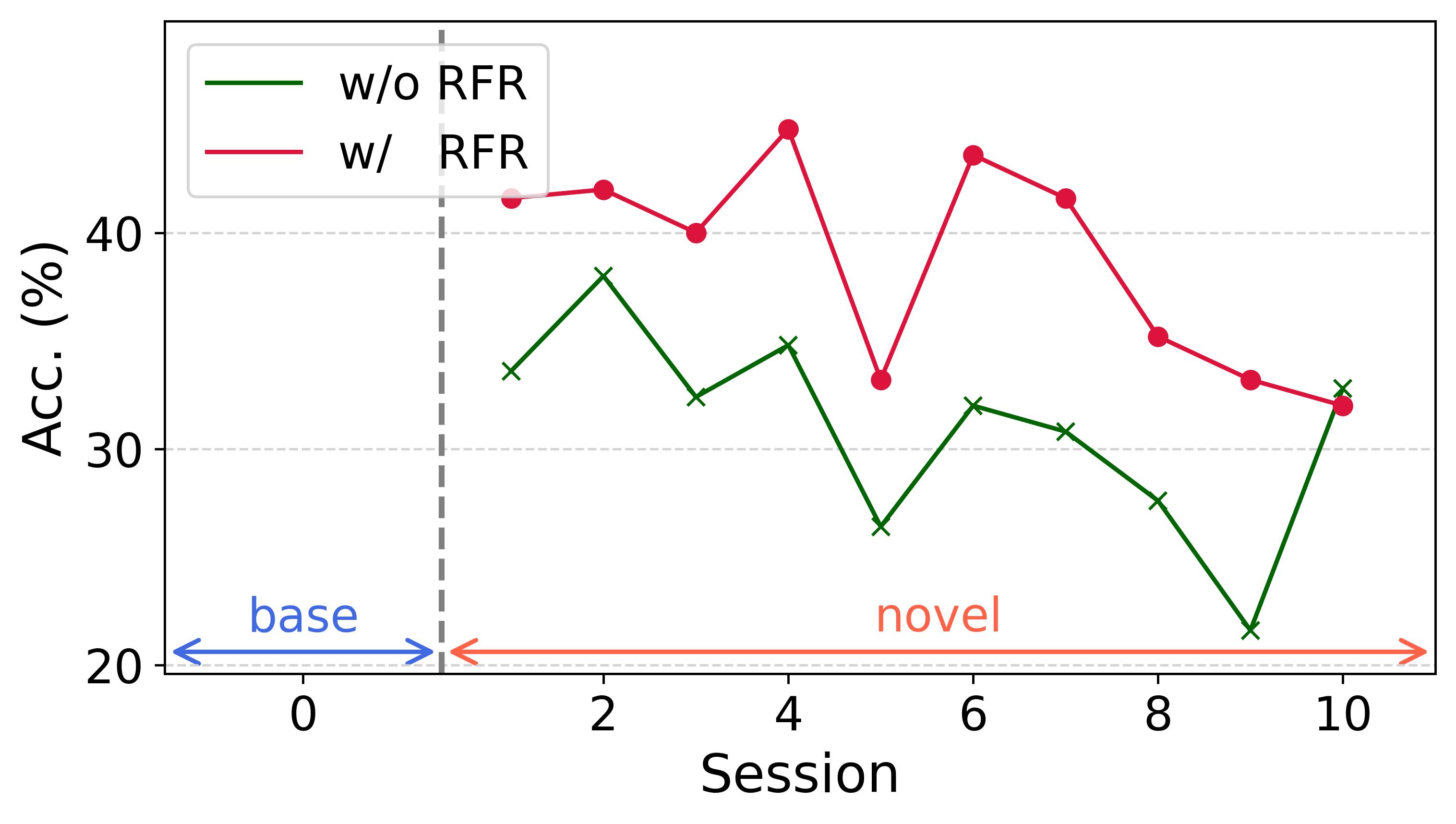}
        \caption{S=5}
    \end{subfigure}
    \begin{subfigure}[b]{0.30\linewidth}
        \includegraphics[width=\linewidth]{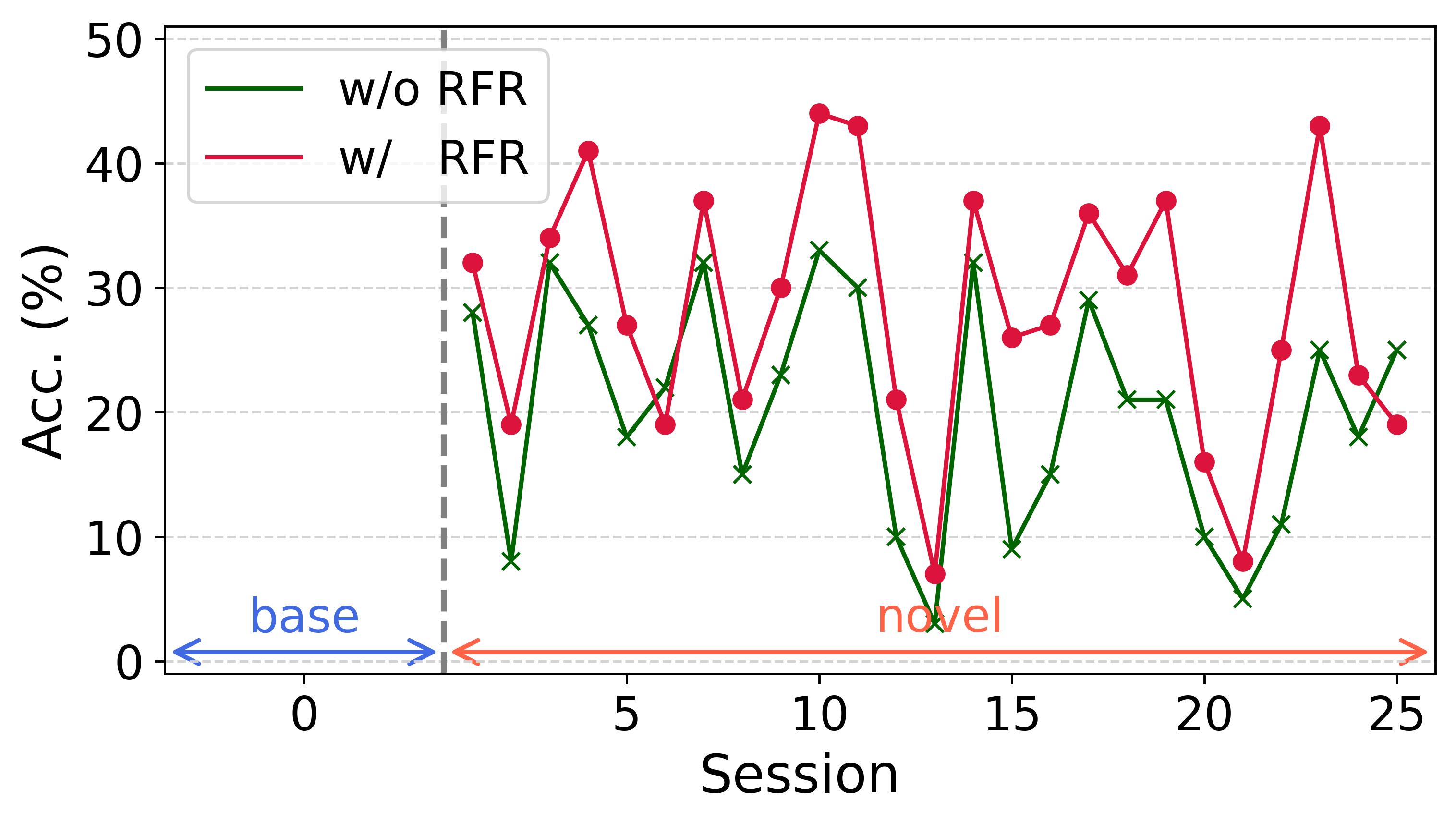}
        \caption{S=2}
    \end{subfigure}
    \caption{Improvements in forward compatibility -- \textit{novel-task} accuracy at each session is shown for UCIR. Results were obtained from the same experiment as in Figure~\ref{fig:acc_freeze_each_phase}.
    }
    \label{fig:acc_freeze_each_novel}
    \centering
    \begin{subfigure}[b]{0.30\linewidth}
        \includegraphics[width=\linewidth]{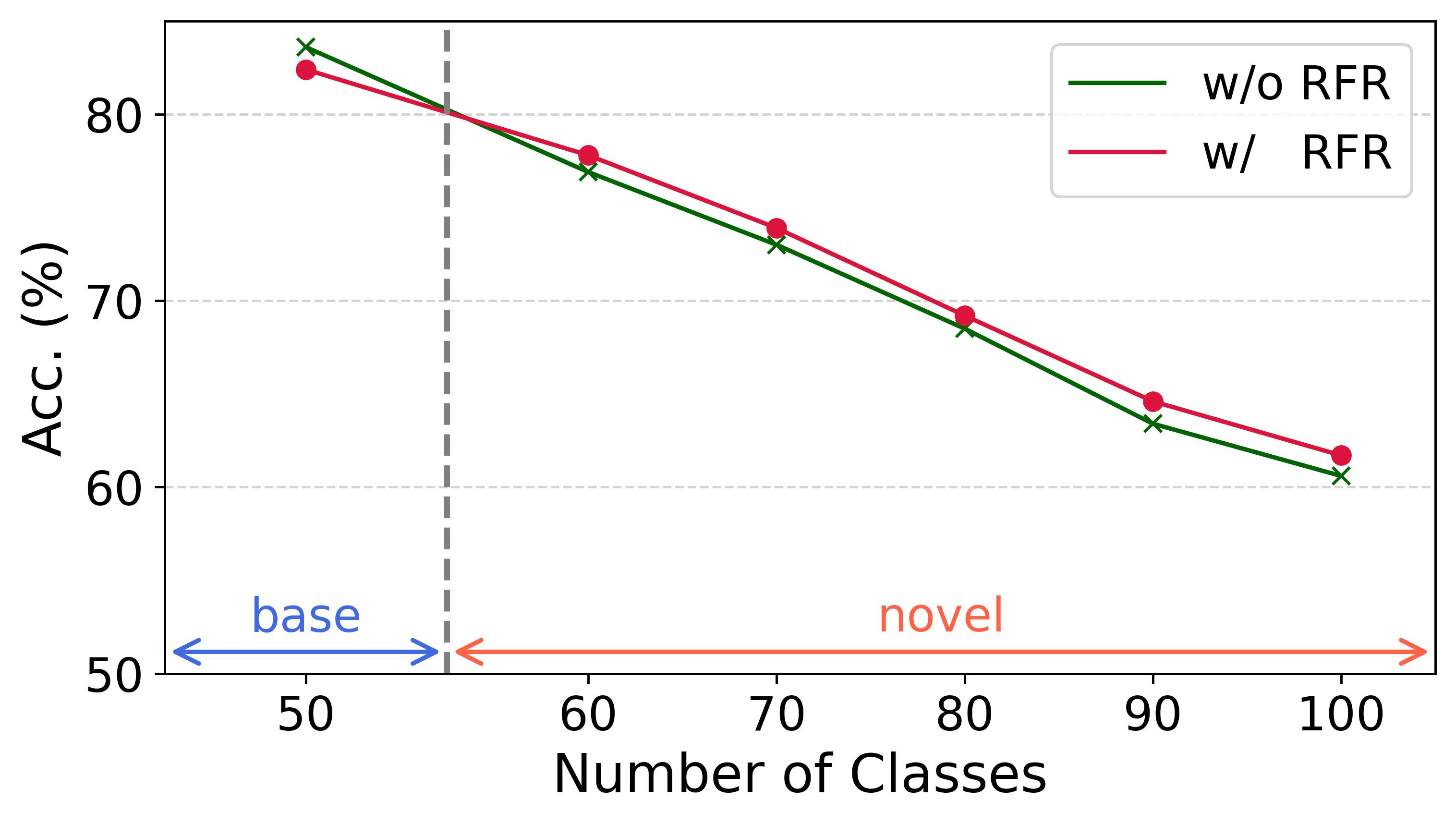}
        \caption{S=10}
    \end{subfigure}
    \begin{subfigure}[b]{0.30\linewidth}
        \includegraphics[width=\linewidth]{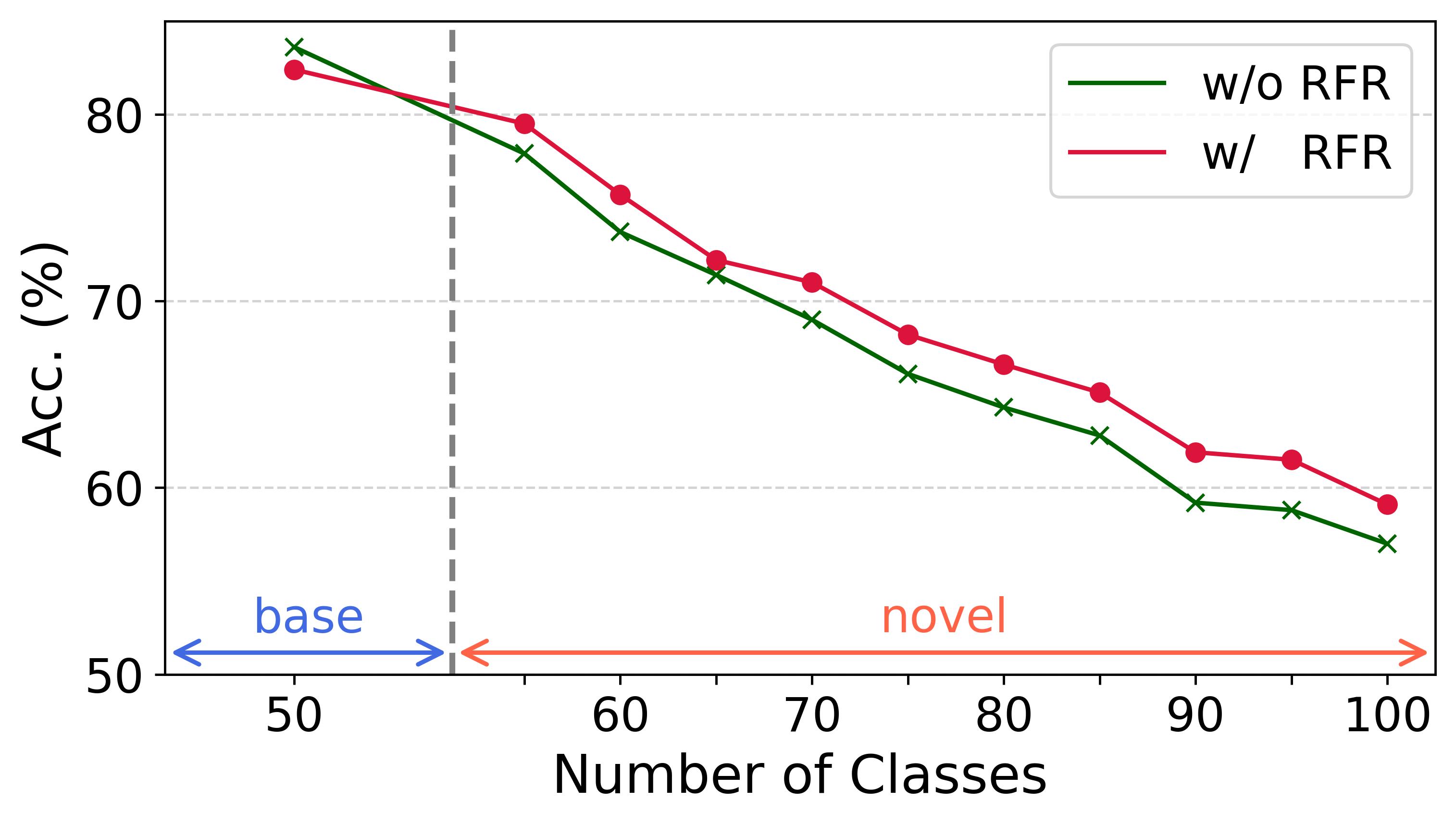}
        \caption{S=5}
    \end{subfigure}
    \begin{subfigure}[b]{0.30\linewidth}
        \includegraphics[width=\linewidth]{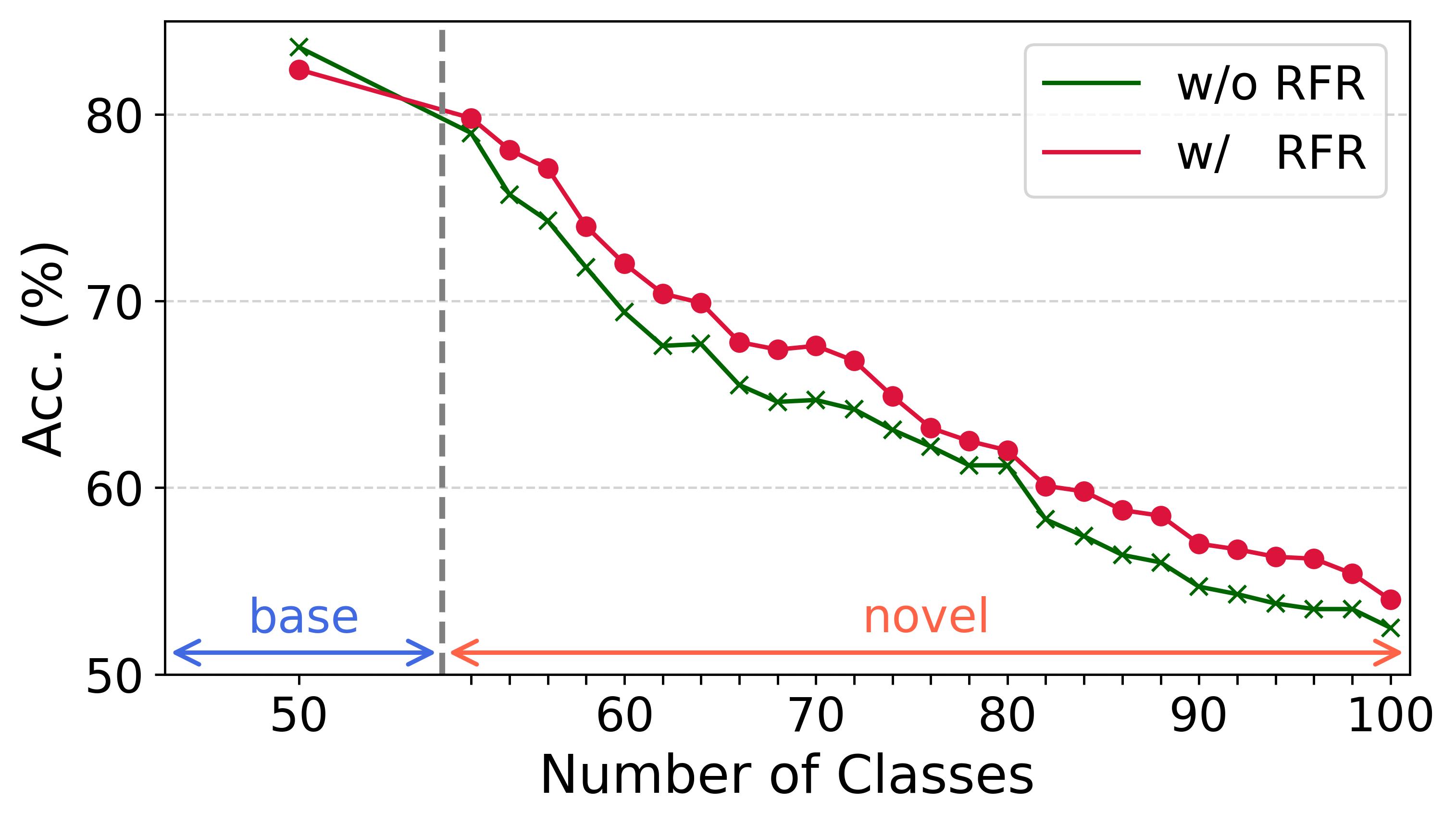}
        \caption{S=2}
    \end{subfigure}
    \caption{Improvements in forward compatibility -- \textit{overall} accuracy at each session is shown for UCIR. 
    Two ResNet-18 models are trained with and without RFR for ImageNet-100 dataset, utilizing 50 base classes under different split sizes (a) 10, (b) 5, and (c) 2 for each novel session.
    }
    \label{fig:acc_each_phase}
    \centering
    \begin{subfigure}[b]{0.30\linewidth}
        \includegraphics[width=\linewidth]{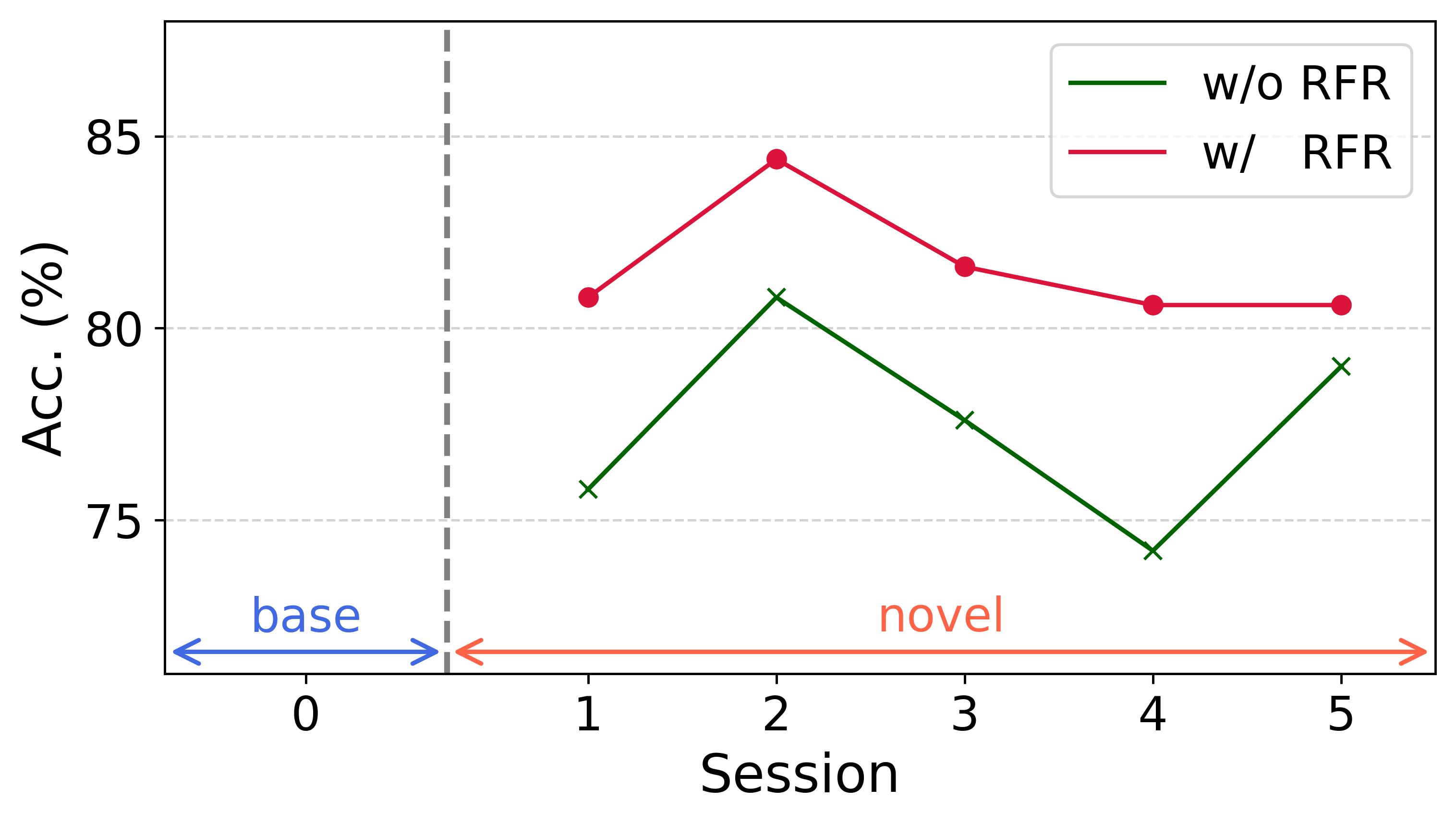}
        \caption{S=10}
    \end{subfigure}
    \begin{subfigure}[b]{0.30\linewidth}
        \includegraphics[width=\linewidth]{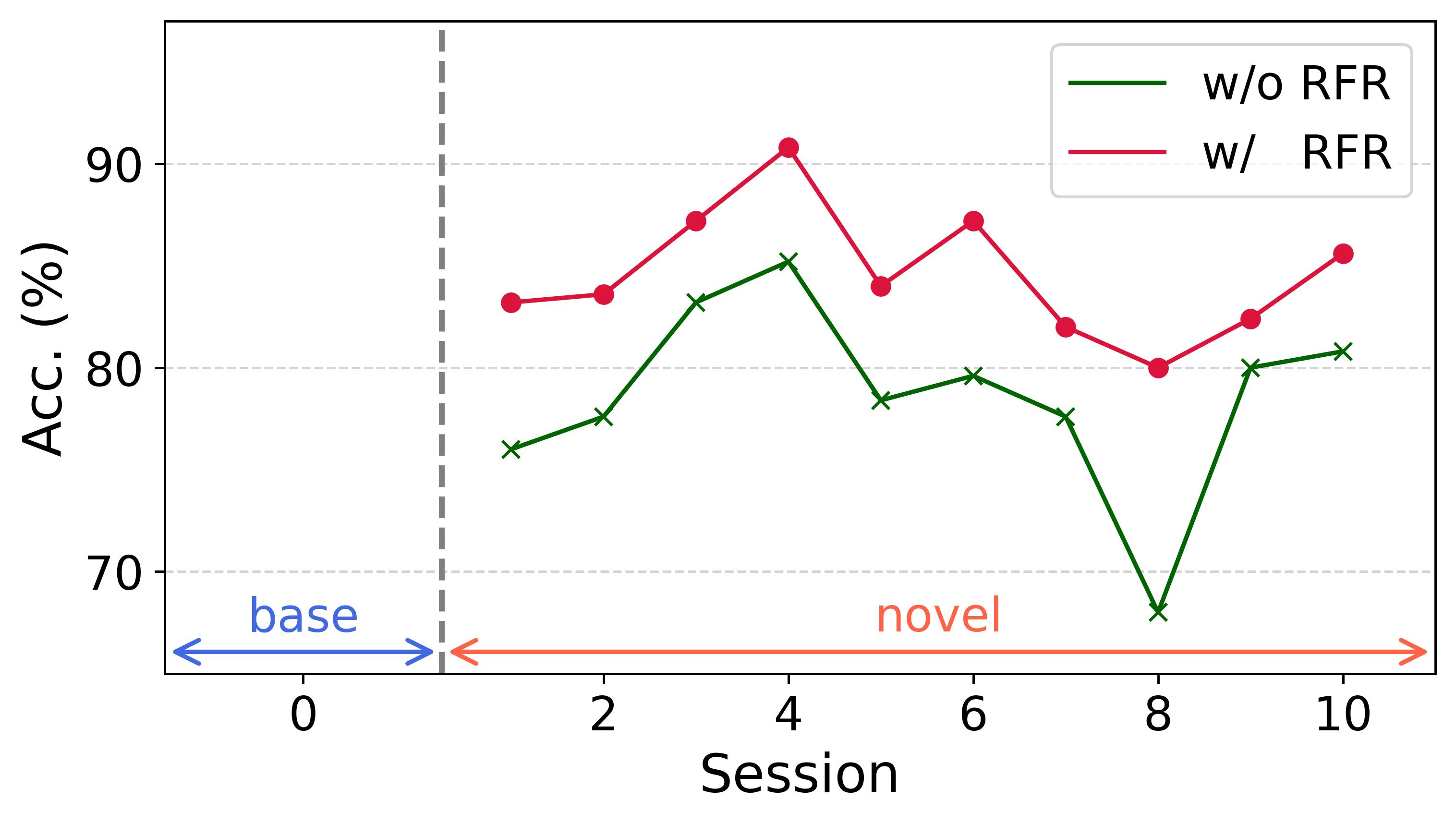}
        \caption{S=5}
    \end{subfigure}
    \begin{subfigure}[b]{0.30\linewidth}
        \includegraphics[width=\linewidth]{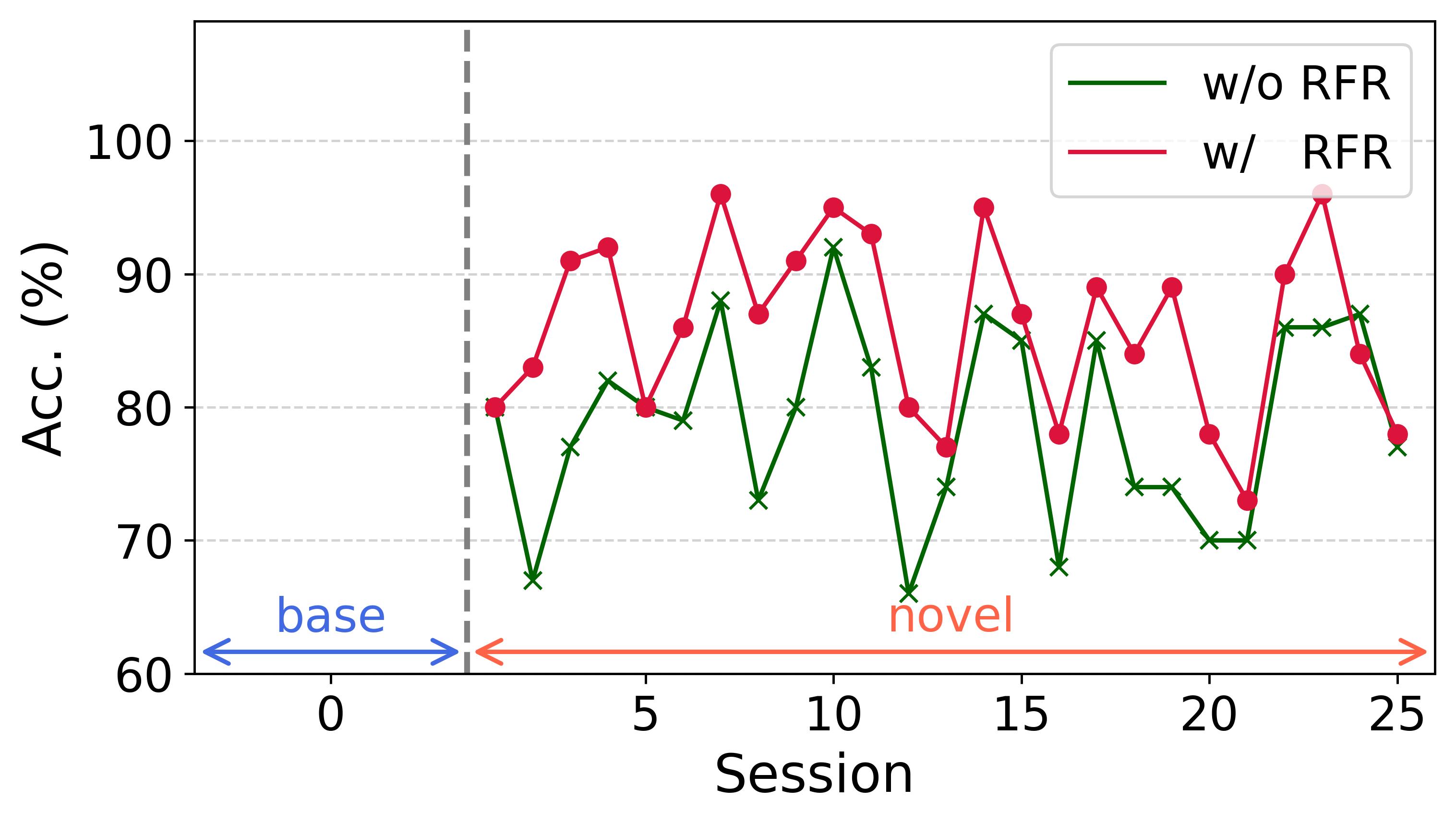}
        \caption{S=2}
    \end{subfigure}
    
    \caption{Improvements in forward compatibility -- \textit{novel-task} accuracy at each session is shown for UCIR. Results were obtained from the same experiment as in Figure~\ref{fig:acc_each_phase}.
    }
    \label{fig:acc_each_novel}
\end{figure*}

\subsection{Settings}
\label{subsec:settings}

\paragraph{Datasets:}
We employ CIFAR-100 and ImageNet-100, two widely adopted benchmark datasets in CIL.
To maintain consistency with prior studies, we follow the standard class orderings proposed in~\citep{rebuffi2017icarl} for all evaluations except for the evaluation of PODNet~\citep{douillard2020podnet} and IL2A~\citep{zhu2021class}, for which we utilize the class orderings defined in PODNet and IL2A, respectively.
To evaluate the CIL methods,
all classes in datasets are divided into multiple tasks.
The initial 50 classes are designated for the base task, while the remaining classes are split into novel tasks, where the size of each split is either 10, 5, or 2. In this study, we denote split sizes of 10, 5, and 2 as S=10, S=5, and S=2, respectively.
To pre-process the datasets, we sequentially apply two simple augmentations: random cropping followed by resizing to the original dimensions (32x32 for the CIFAR dataset and 224x224 for the ImageNet dataset). Additionally, there is a 50\% probability of applying random horizontal flipping. Afterward, we normalize the color channels by subtracting the mean values ((0.4914, 0.4822, 0.4465) for CIFAR and (0.485, 0.456, 0.406) for ImageNet) and dividing by the respective standard deviations ((0.247, 0.243, 0.261) for CIFAR and (0.229, 0.224, 0.225) for ImageNet).

\paragraph{Implementation Details:}
In this study, ResNet-18~\citep{he2016deep} is employed as the base model to investigate.
To reproduce the results of iCaRL~\citep{rebuffi2017icarl}, LwF~\citep{li2017learning}, SI~\citep{zenke2017continual}, EEIL~\citep{castro2018end}, MAS~\citep{aljundi2018memory}, RWalk~\citep{chaudhry2018riemannian}, BiC~\citep{wu2019large}, IL2M~\citep{belouadah2019il2m}, and UCIR~\citep{hou2019learning}, we utilize the open-source codebase provided by FACIL~\citep{masana2022class} and retain their original hyperparameters unchanged.
For PODNet~\citep{douillard2020podnet}, IL2A~\citep{zhu2021class}, and AFC~\citep{kang2022class}, we employ their own open-source codebases and retain their original hyperparameters unchanged.
To incorporate RFR into the aforementioned methods,
we only modify their respective loss functions during the training of the base session, as specified by Eq.~(\ref{eq:loss}). To determine the optimal RFR loss coefficient, $\alpha$, we performed a grid search over $\alpha \in \{0.05, 0.10, 0.15, 0.20\}$.
Meanwhile, the remaining configurations, such as the loss functions applied during the training of novel sessions and the default hyper-parameter settings, remain unchanged.
To evaluate the performance of the models, we adopt the standard metric of \textit{average incremental accuracy} (AIC) proposed in~\citep{rebuffi2017icarl}.

\subsection{Superiority in forward compatibility}
\label{subsec:forward_compatibility}

To investigate forward compatibility improvement, two models are trained for a base task, one with RFR and one without RFR. Then, we compare the novel task performance while keeping the feature extractors frozen and training only the classification heads for the respective novel tasks.
The results are shown in Figure~\ref{fig:acc_freeze_each_phase}, and they demonstrate that our method yields improved novel task AIC performance across all three cases, with enhancements of 2.50\%, 2.63\%, and 2.59\% for split sizes of 10, 5, and 2, respectively.
To delve deeper into this improvement, we analyze the performance of each individual novel task. As shown in Figure~\ref{fig:acc_freeze_each_novel}, our method leads to enhanced performance for the majority of novel tasks. Moreover, the average performance in novel tasks demonstrates significant enhancements, with improvements of 5.93\%, 7.11\%, and 7.77\% for split sizes of 10, 5, and 2, respectively.
These results strongly indicate that our method is effective in promoting forward compatible representations.
Additionally, we conduct a similar analysis, this time without freezing of feature extractors. The results presented in Figure~\ref{fig:acc_each_phase} and Figure~\ref{fig:acc_each_novel} exhibit similar performance enhancements, thereby substantiating superiority of RFR's forward compatibility.

\subsection{Mitigation in catastrophic forgetting}
\label{subsec:forgetting}

\begin{figure*}[htbp]
    \centering
    \begin{subfigure}[b]{0.30\linewidth}
        \includegraphics[width=\linewidth]{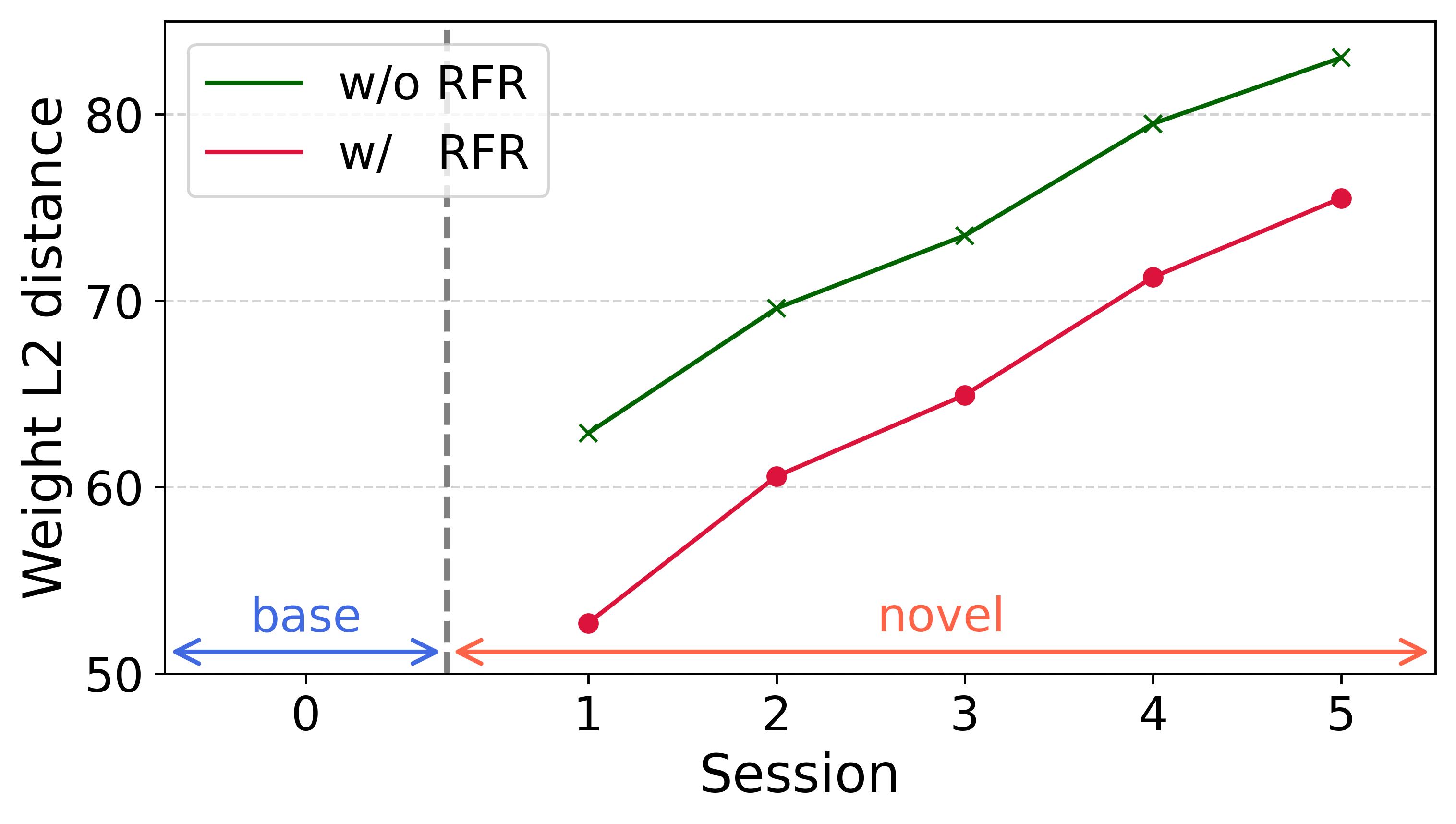}
        \caption{S=10}
    \end{subfigure}
    \begin{subfigure}[b]{0.30\linewidth}
        \includegraphics[width=\linewidth]{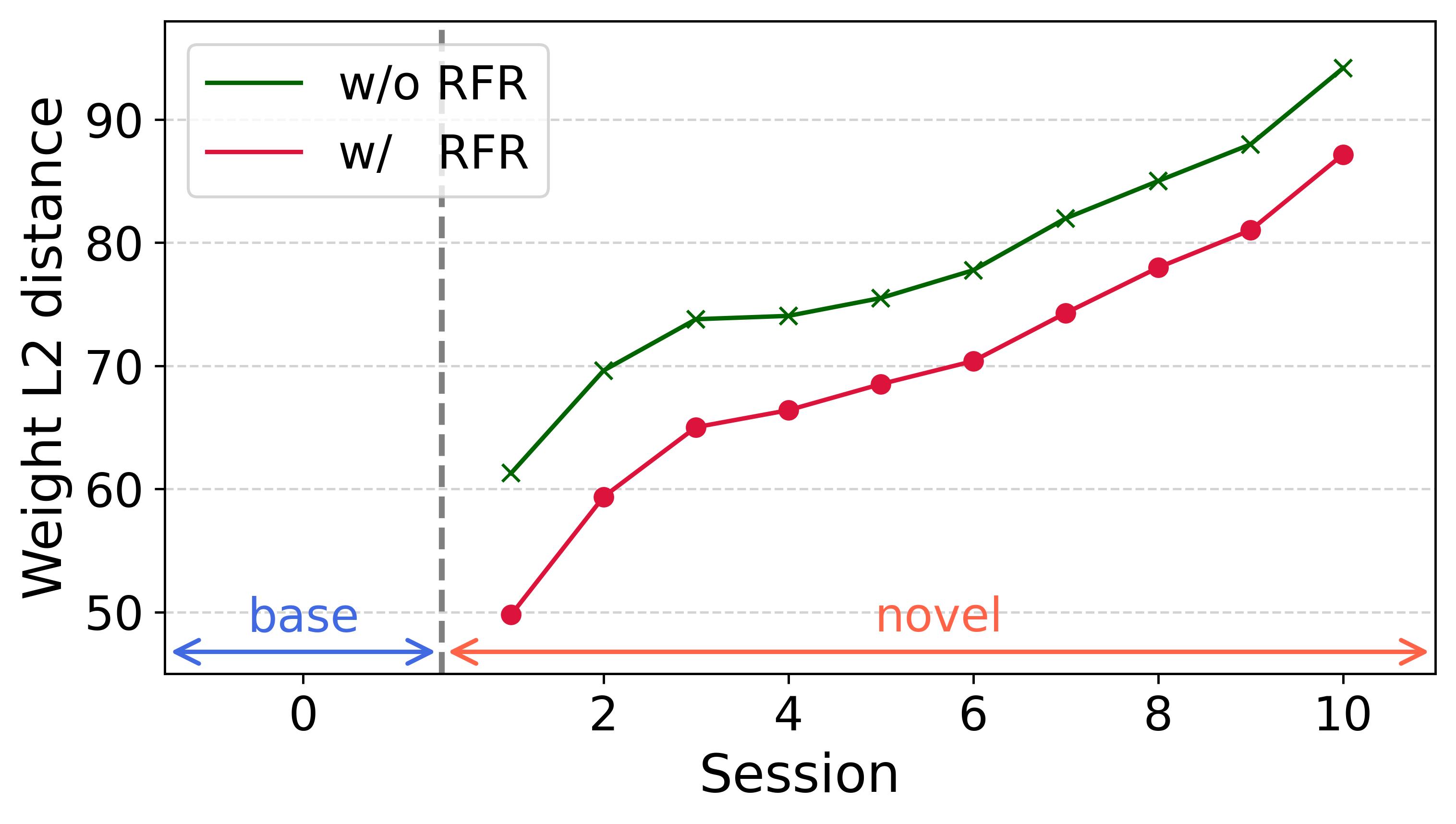}
        \caption{S=5}
    \end{subfigure}
    \begin{subfigure}[b]{0.30\linewidth}
        \includegraphics[width=\linewidth]{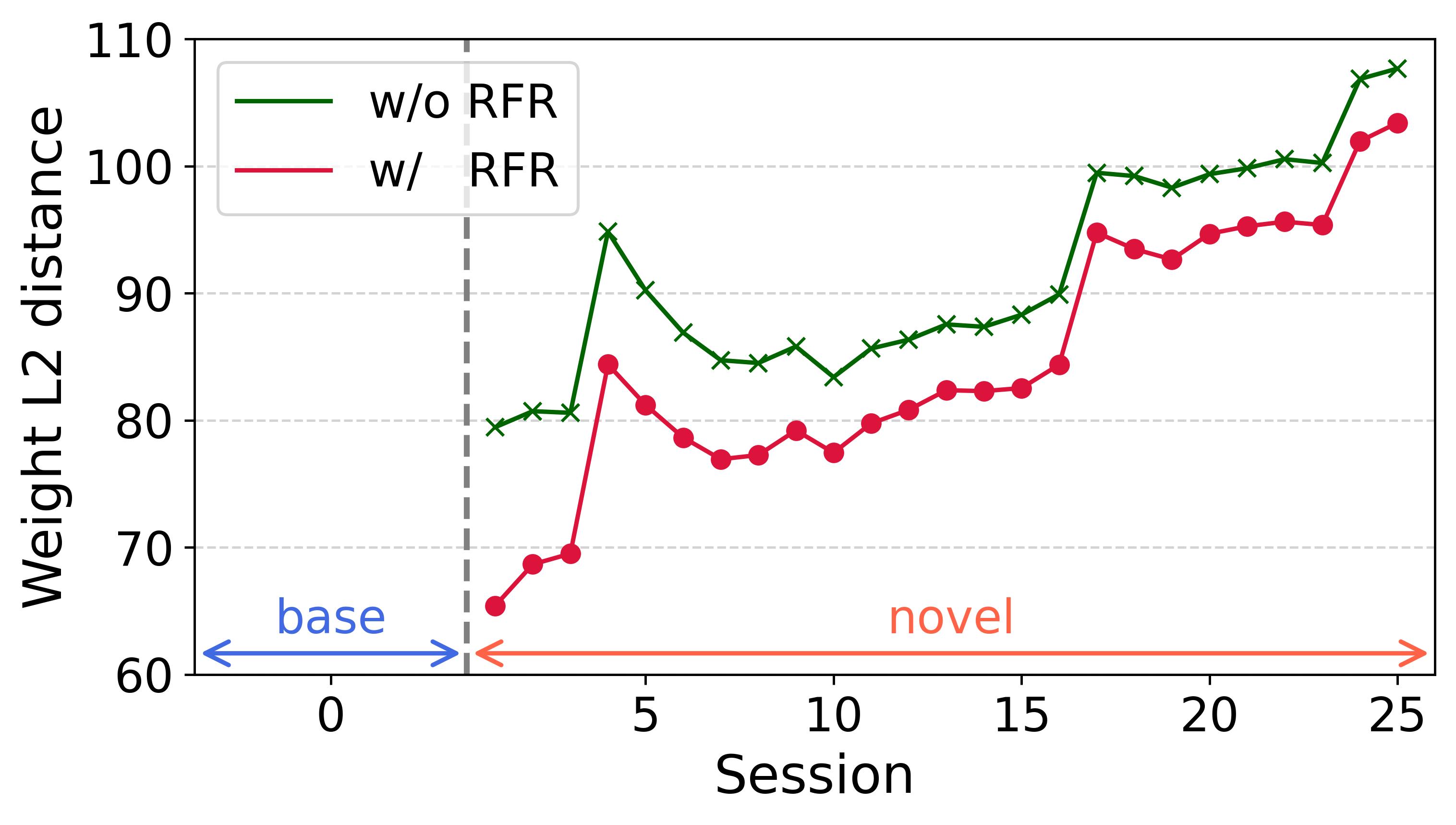}
        \caption{S=2}
    \end{subfigure}
    \caption{Weight change from the base session for UCIR.
    Two ResNet-18 models are trained with and without RFR for ImageNet-100 dataset, utilizing 50 base classes under different split sizes (a) 10, (b) 5, and (c) 2 for each novel session.
    }
    \label{fig:weight_changes_from_base}
    \centering
    \begin{subfigure}[b]{0.30\linewidth}
        \includegraphics[width=\linewidth]{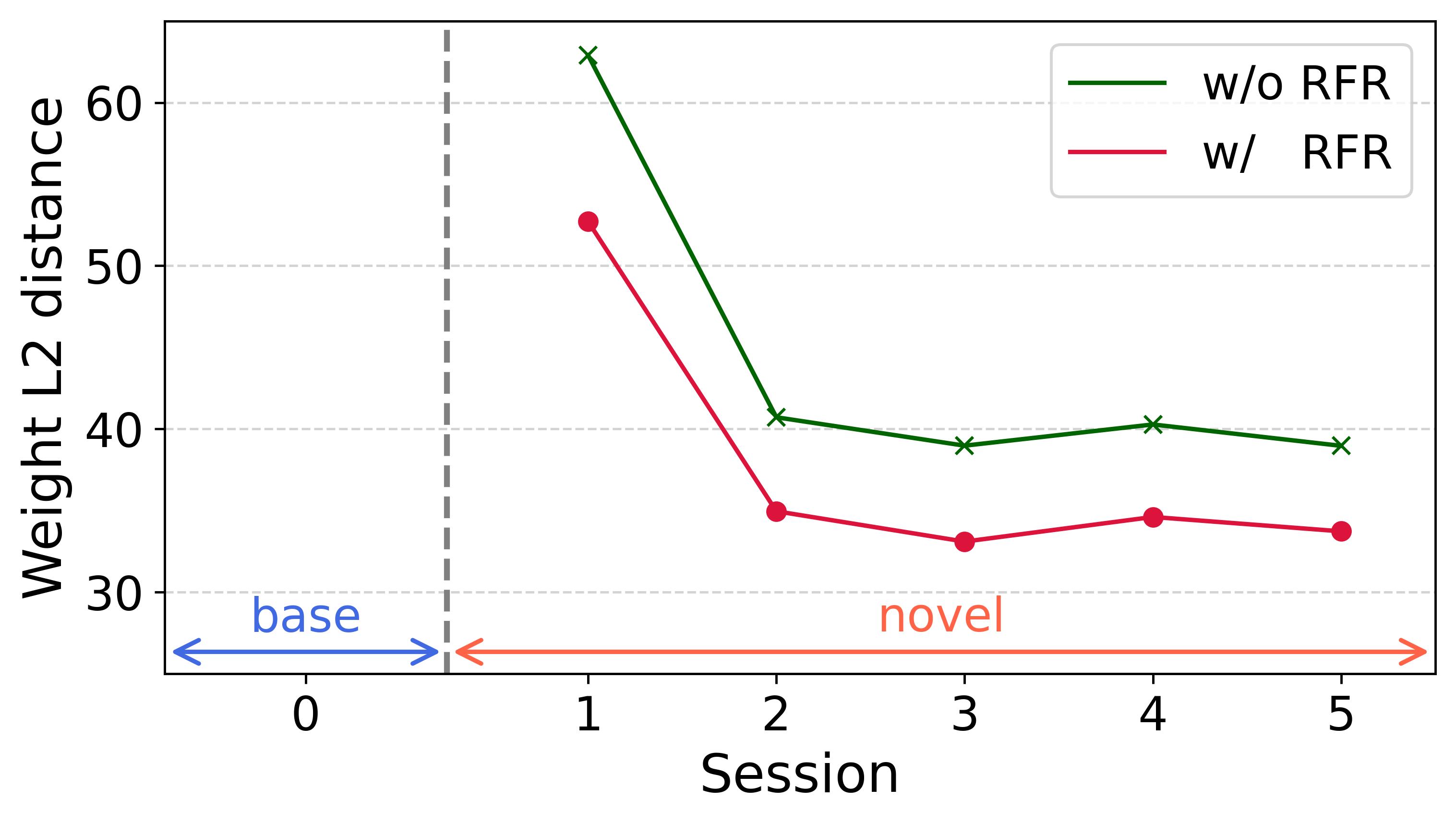}
        \caption{S=10}
    \end{subfigure}
    \begin{subfigure}[b]{0.30\linewidth}
        \includegraphics[width=\linewidth]{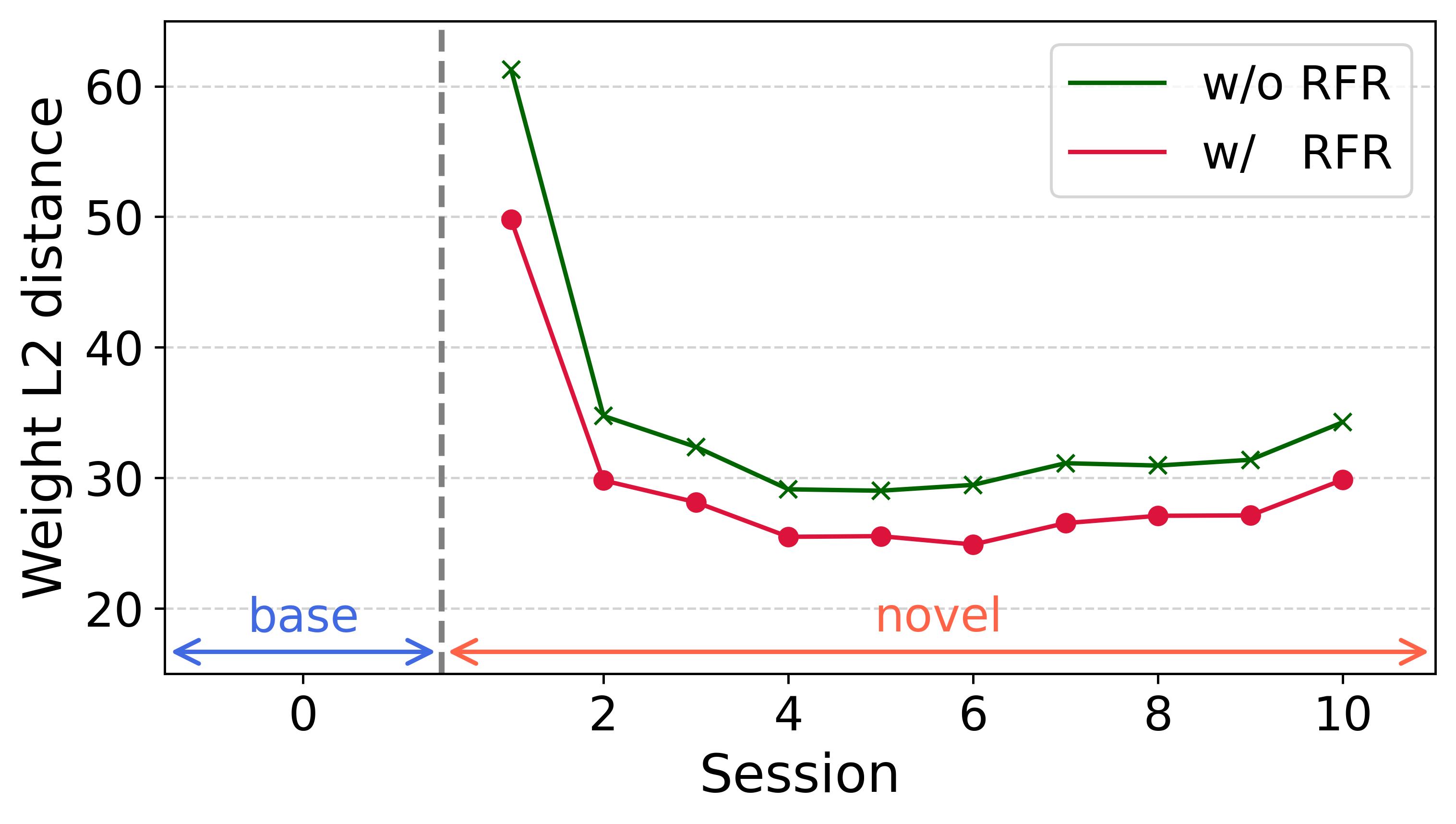}
        \caption{S=5}
    \end{subfigure}
    \begin{subfigure}[b]{0.30\linewidth}
        \includegraphics[width=\linewidth]{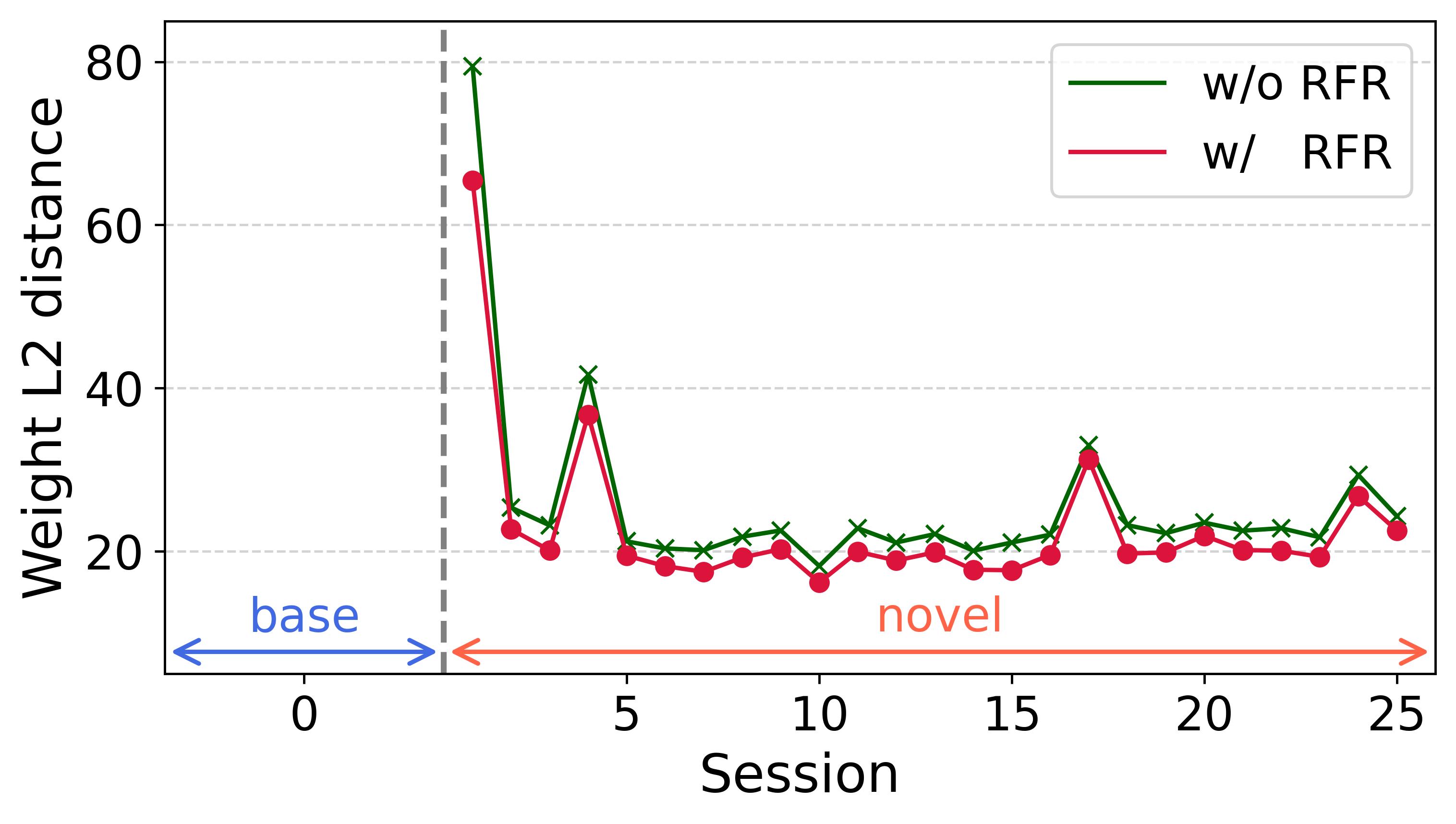}
        \caption{S=2}
    \end{subfigure}
    
    \caption{Weight change from the immediate previous session for UCIR.
    Two ResNet-18 models are trained with and without RFR for ImageNet-100 dataset, utilizing 50 base classes under different split sizes (a) 10, (b) 5, and (c) 2 for each novel session.
    }
    \label{fig:weight_changes_from_prev}
    \centering
    \begin{subfigure}[b]{0.30\linewidth}
        \includegraphics[width=\linewidth]{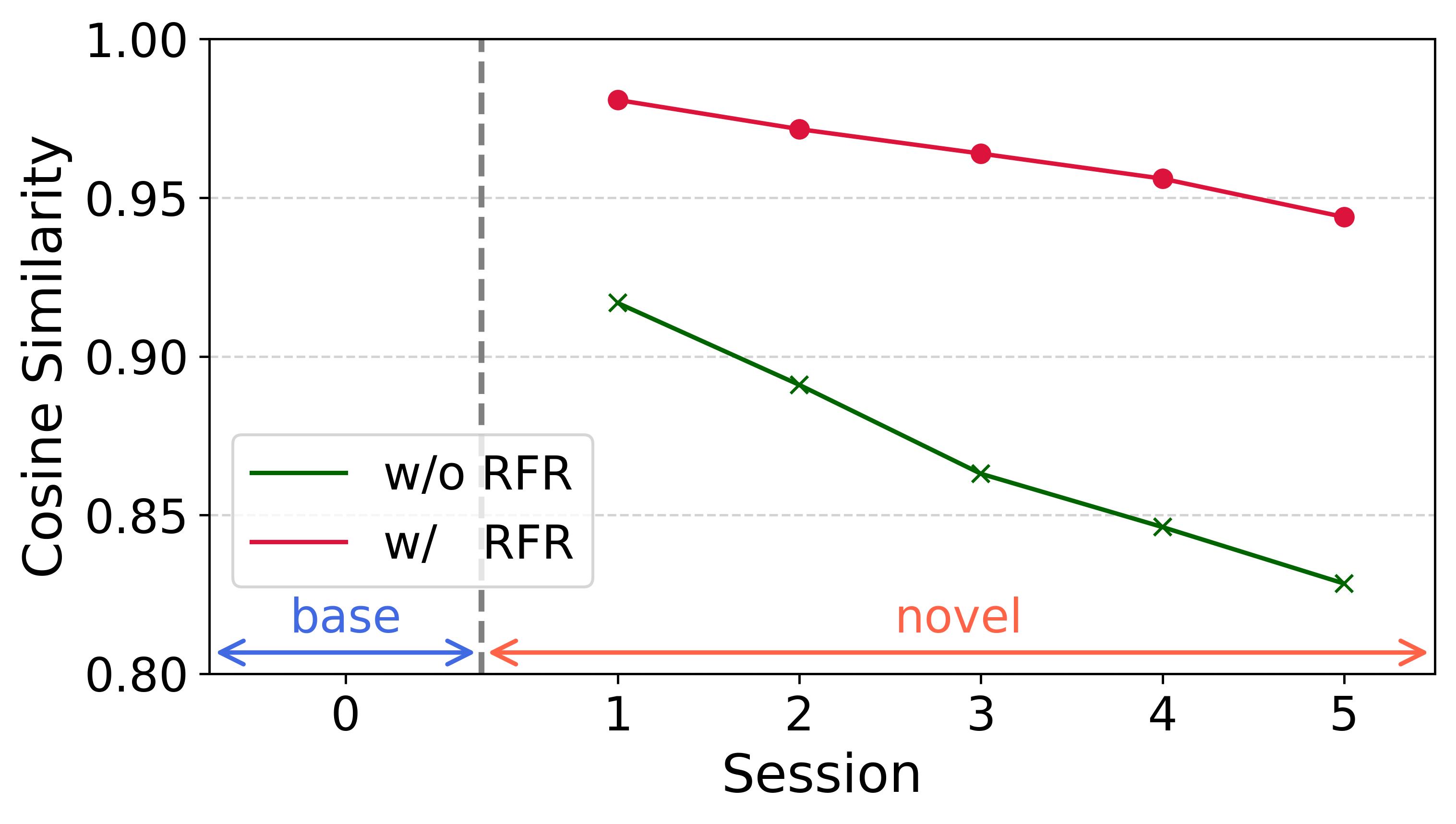}
        \caption{S=10}
    \end{subfigure}
    \begin{subfigure}[b]{0.30\linewidth}
        \includegraphics[width=\linewidth]{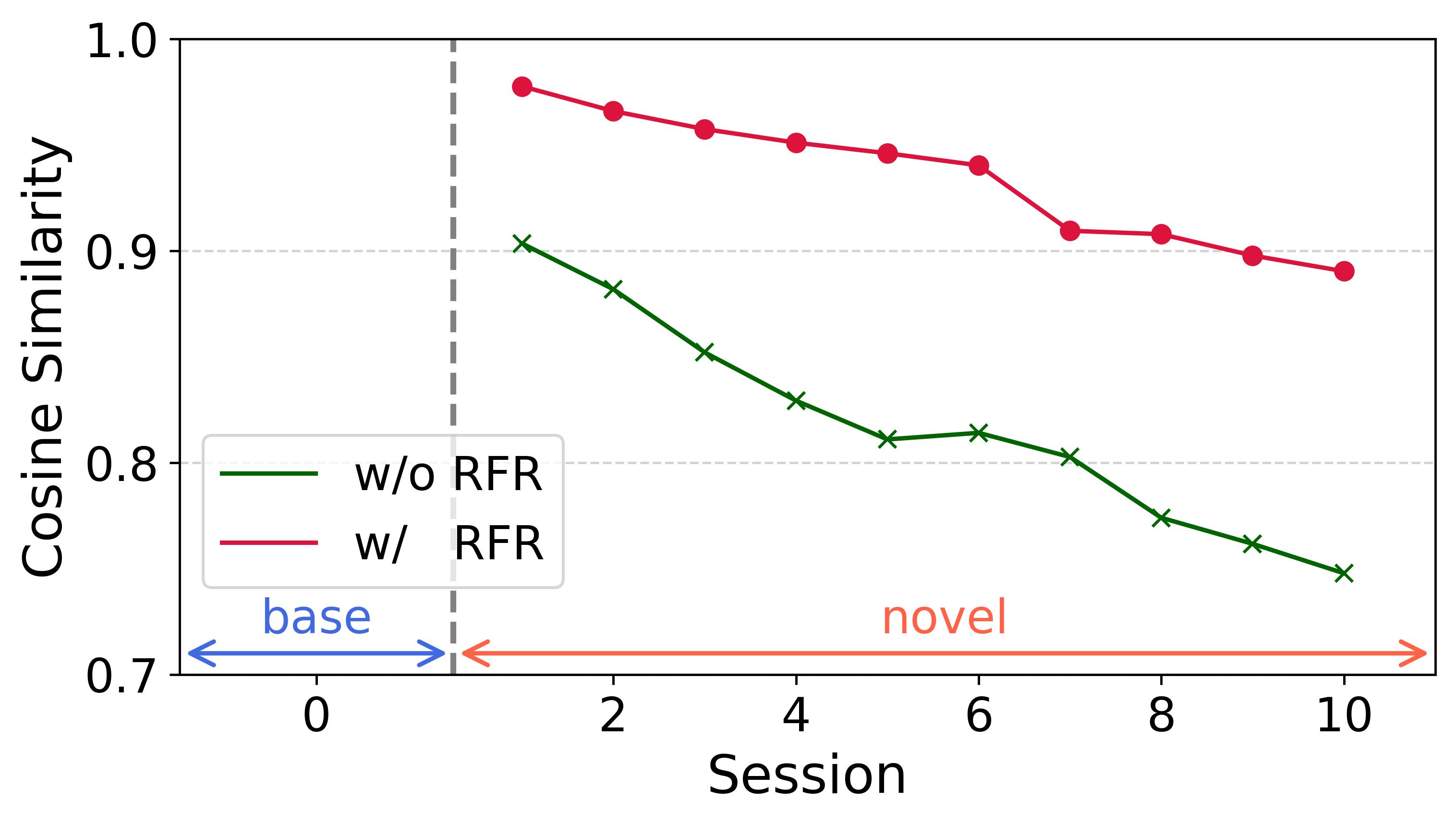}
        \caption{S=5}
    \end{subfigure}
    \begin{subfigure}[b]{0.30\linewidth}
        \includegraphics[width=\linewidth]{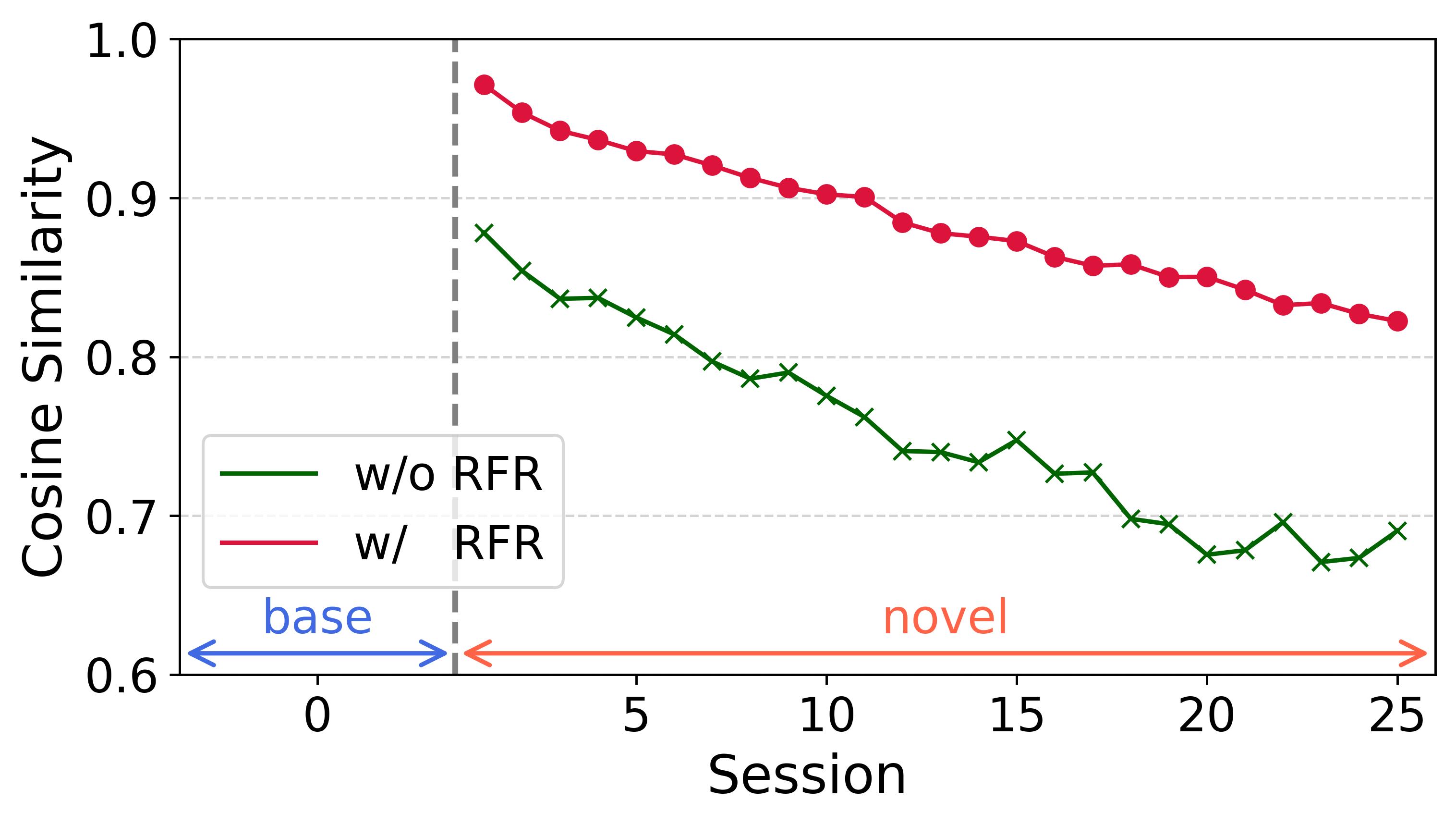}
        \caption{S=2}
    \end{subfigure}
    \caption{Cosine similarity in representation with respect to the base session for UCIR. 
    Two ResNet-18 models are trained with and without RFR for ImageNet-100 dataset, utilizing 50 base classes under different split sizes (a) 10, (b) 5, and (c) 2 for each novel session.}
    \label{fig:similarity_with_base}
    \centering
    \begin{subfigure}[b]{0.30\linewidth}
        \includegraphics[width=\linewidth]{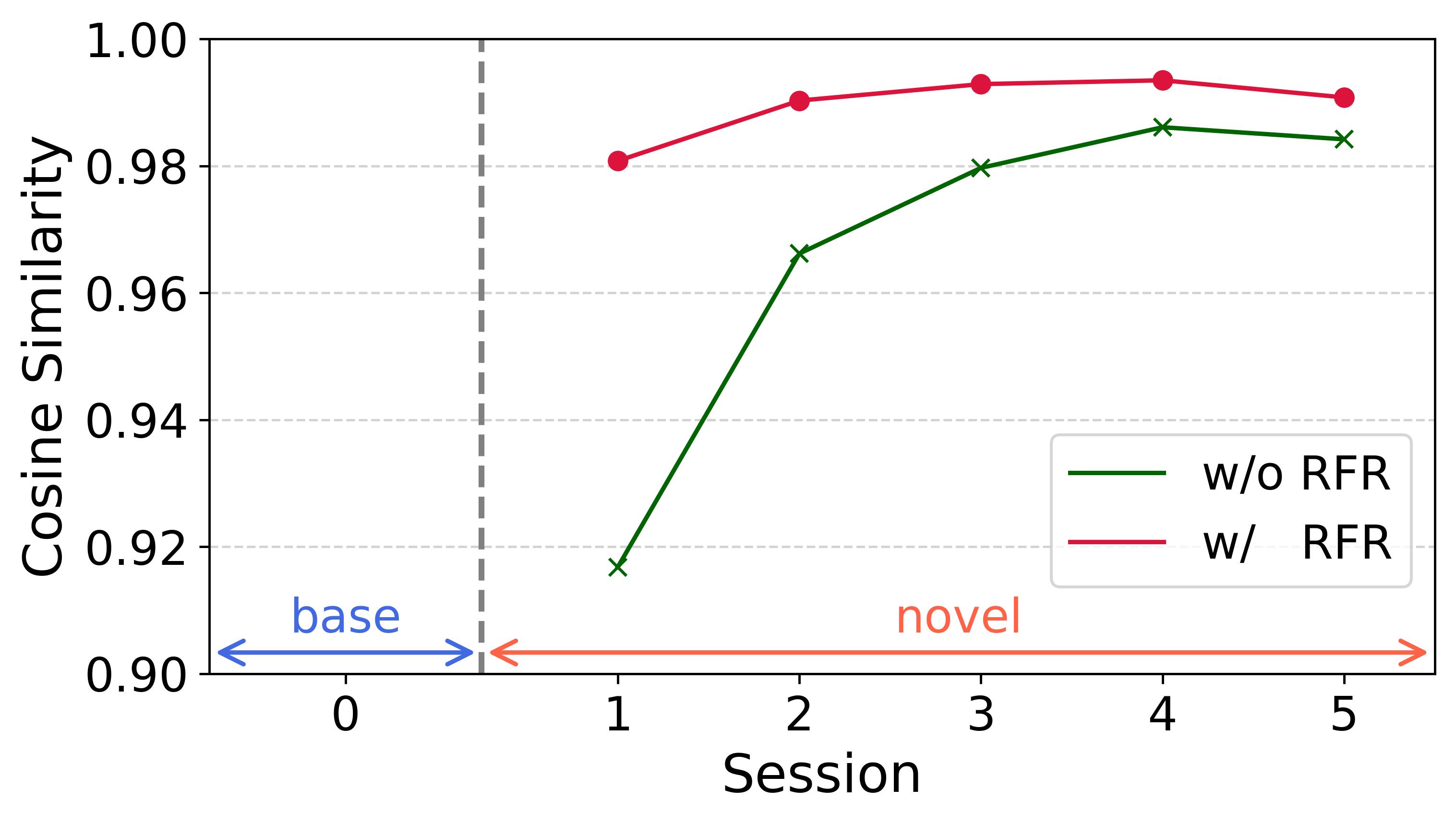}
        \caption{S=10}
    \end{subfigure}
    \begin{subfigure}[b]{0.30\linewidth}
        \includegraphics[width=\linewidth]{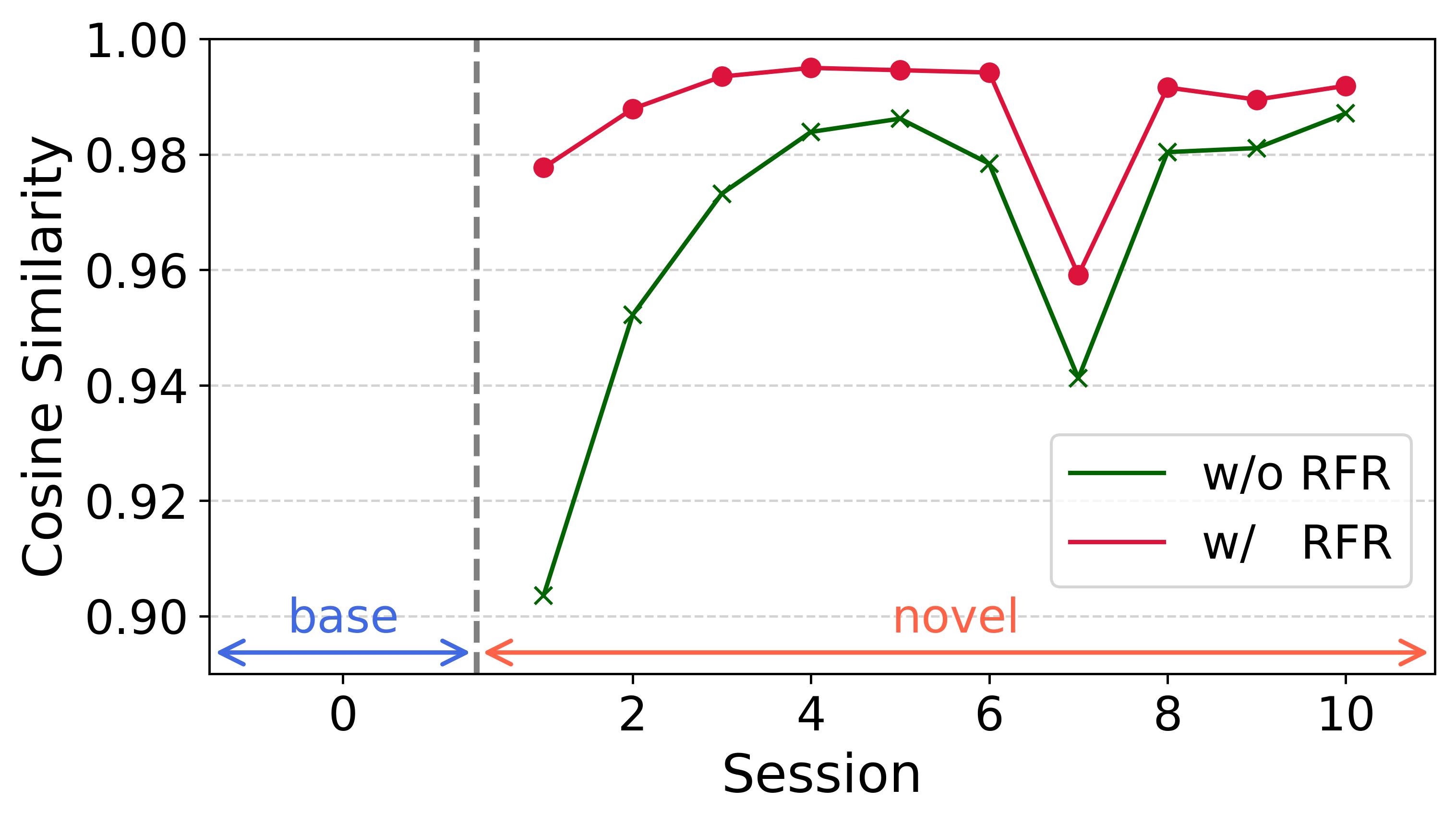}
        \caption{S=5}
    \end{subfigure}
    \begin{subfigure}[b]{0.30\linewidth}
        \includegraphics[width=\linewidth]{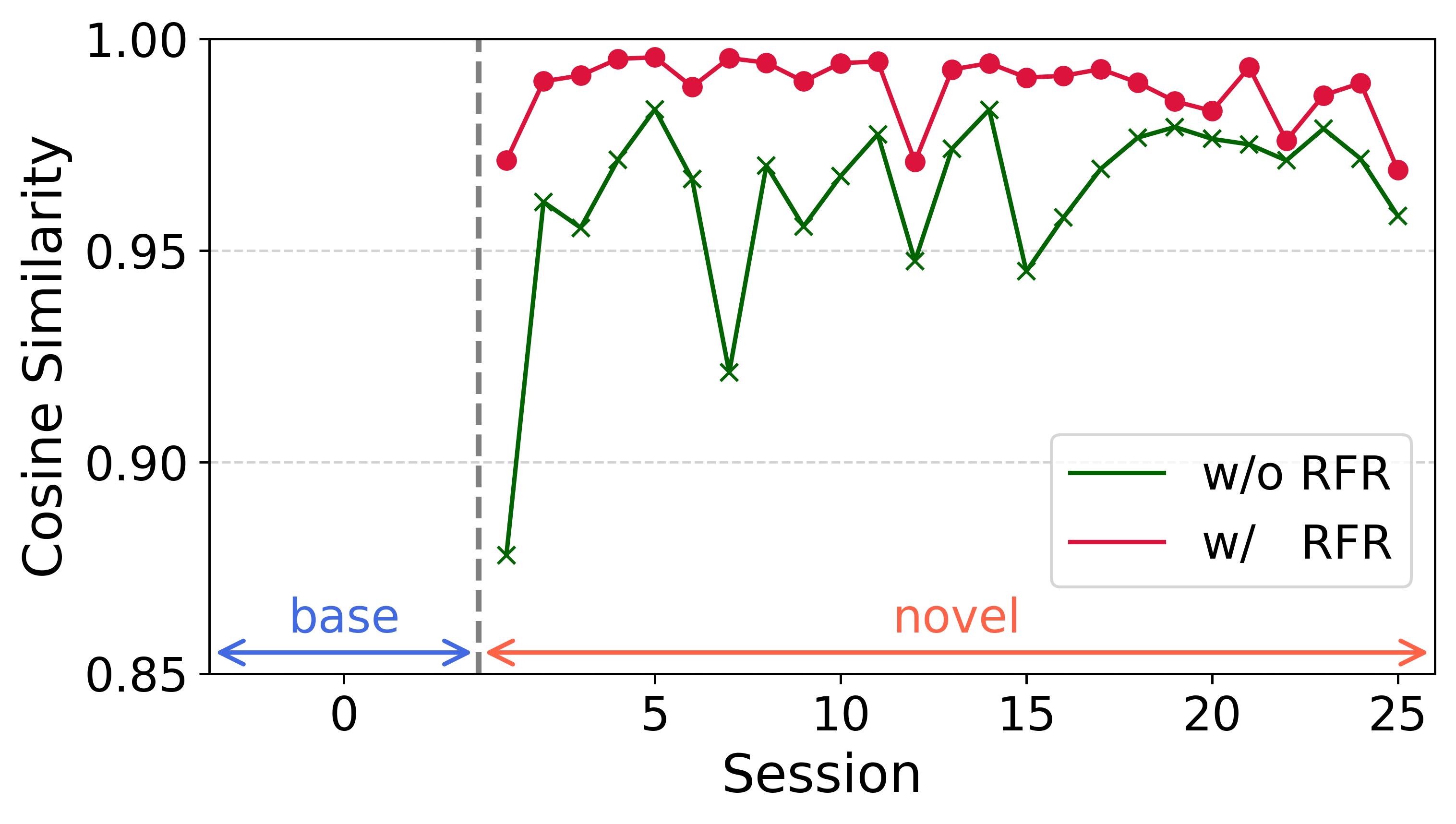}
        \caption{S=2}
    \end{subfigure}
    \caption{Cosine similarity in representation with respect to the immediate previous session for UCIR. 
    Two ResNet-18 models are trained with and without RFR for ImageNet-100 dataset, utilizing 50 base classes under different split sizes (a) 10, (b) 5, and (c) 2 for each novel session.}
    \label{fig:similarity_with_prev}
\end{figure*}

To analyze improvement in catastrophic forgetting, we conduct three investigations.
First, we examine the influence of our method on the weight changes of the feature extractor.
As novel sessions progress, the $L_2$-weight distance between the feature extractor learned in the base session and that in the novel session increases, as demonstrated in Figure~\ref{fig:weight_changes_from_base}.
However, with the integration of RFR, the weight distances are significantly reduced across all sessions.
The average reductions in weight distances are 8.70, 8.13, and 6.79 for split sizes of 10, 5, and 2, respectively.
Additionally, Figure~\ref{fig:weight_changes_from_prev} reveals that the weight distance with the feature extractor from the immediate previous session also increases less when our method is employed, showing average reduction of 6.55, 4.95, and 2.99 for split sizes of 10, 5, and 2, respectively.

Second, we investigate the impact of our method on cosine similarity in representations produced from the validation dataset of the base task.
As shown in Figure~\ref{fig:similarity_with_base}, our method leads to an increase in the similarity between the representations learned in the base session and those in novel sessions. The increase in representation similarity across all sessions is significant, with average increase of 0.09, 0.12, and 0.13 for split sizes of 10, 5, and 2, respectively.
Moreover, Figure~\ref{fig:similarity_with_prev} demonstrates that the similarity with representations from the immediate previous session also increases, with average increase of 0.02, 0.02, and 0.03 for split sizes of 10, 5, and 2, respectively.

\begin{figure*}[ht!]
    \centering
    \begin{subfigure}{0.30\linewidth}
        \includegraphics[width=\linewidth]{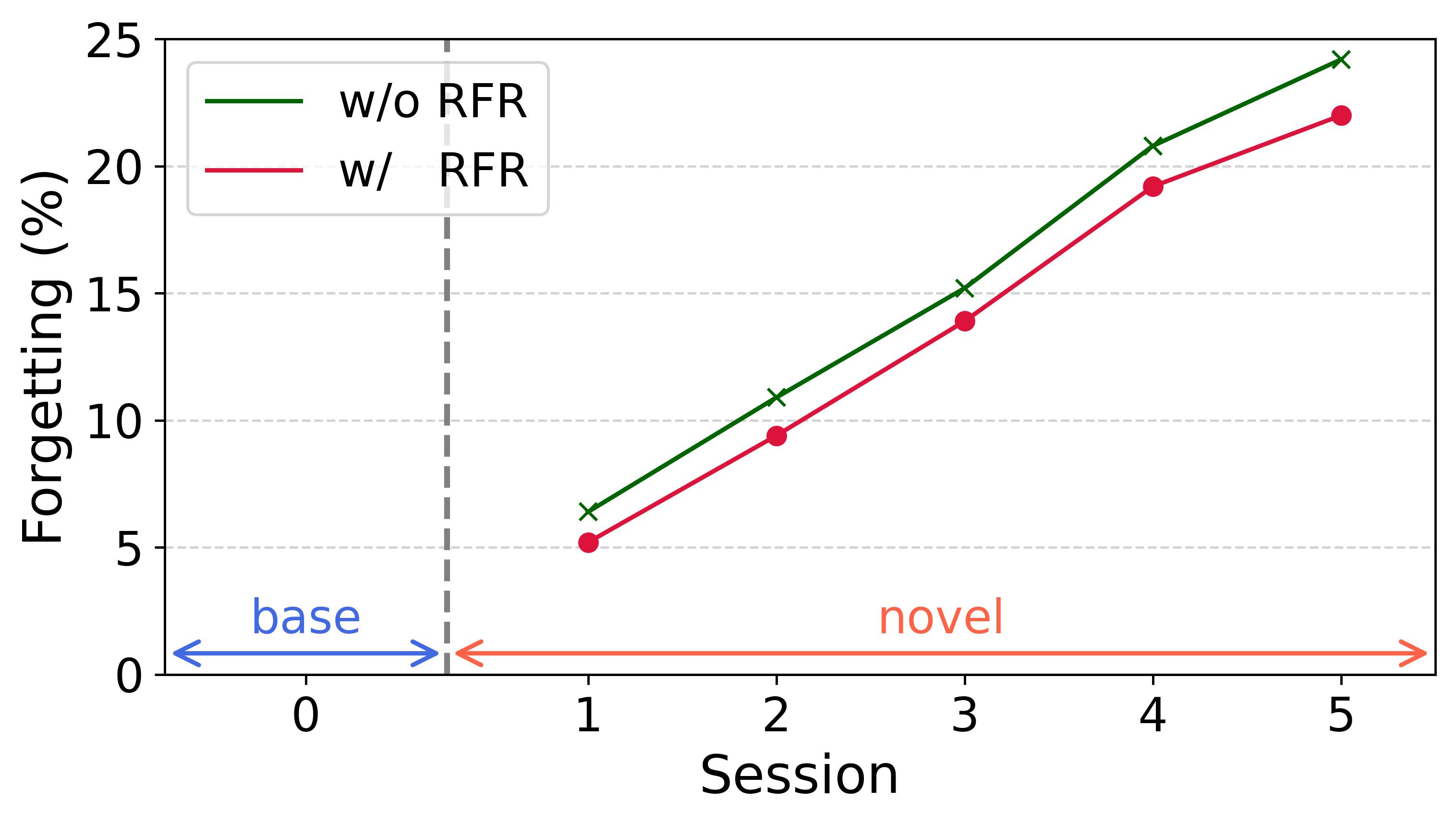}
        \caption{S=10}
    \end{subfigure}
    \begin{subfigure}{0.30\linewidth}
        \includegraphics[width=\linewidth]{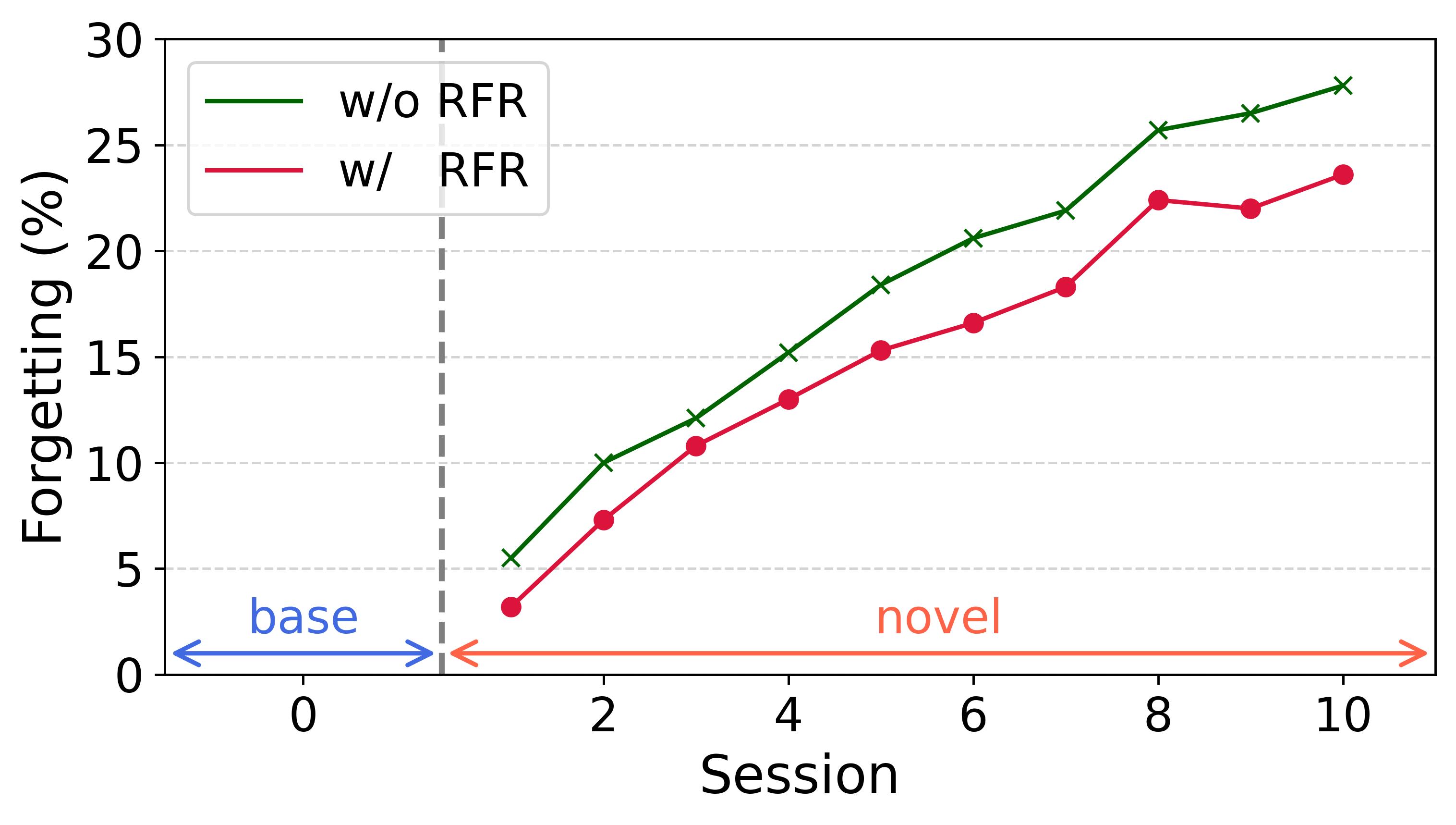}
        \caption{S=5}
    \end{subfigure}
    \begin{subfigure}{0.30\linewidth}
        \includegraphics[width=\linewidth]{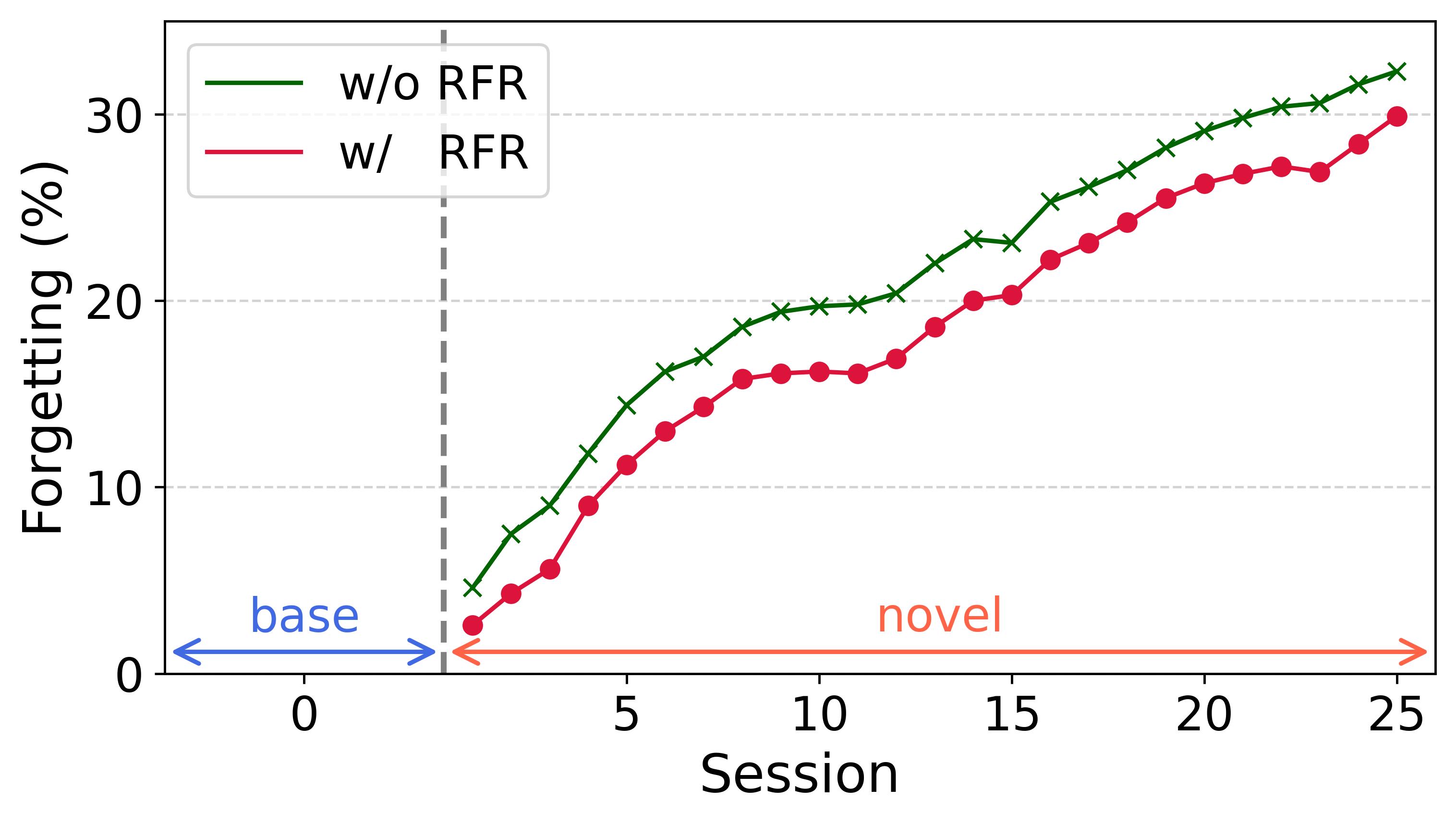}
        \caption{S=2}
    \end{subfigure}
    \caption{Catastrophic forgetting in the base task for UCIR.
    Two ResNet-18 models are trained with and without RFR for ImageNet-100 dataset, utilizing 50 base classes under different split sizes (a) 10, (b) 5, and (c) 2 for each novel session.
    }
    \label{fig:catastrophic_forgetting}
\end{figure*}

Finally, we evaluate the actual impact on catastrophic forgetting in the base task, measured by the performance drop on the base task~\citep{liu2020mnemonics,mittal2021essentials}.
Figure~\ref{fig:catastrophic_forgetting} clearly demonstrates that our method leads to decreased catastrophic forgetting in the base task across all sessions, with an average reduction of 1.30\%, 2.84\%, and 2.90\% for split sizes of 10, 5, and 2, respectively.

\subsection{Consistent improvements over eleven existing methods}
\label{subsec:sota_experiment}

\begin{table*}[ht!]
\caption{\label{tab:sota_experiment1} Performance improvements by RFR. All methods in this table follow the same class orderings as originally proposed in iCaRL~\citep{rebuffi2017icarl}.
}
\centering
\begin{tabular}{@{}lrrrrrr@{}}
\toprule
Method        & \multicolumn{3}{c}{CIFAR100 (B=50)}                                                               & \multicolumn{3}{c}{ImageNet-100 (B=50)}                                                           \\ \cmidrule(lr){2-4} \cmidrule(l){5-7}
              & S=10       & S=5        & S=2        & S=10       & S=5        & S=2        \\ \midrule
iCaRL~\citep{rebuffi2017icarl}         & 49.83$_{\pm2.33}$               & 46.86$_{\pm2.38}$              & 44.69$_{\pm1.97}$              & 53.77$_{\pm1.68}$               & 48.41$_{\pm8.70}$              & 49.78$_{\pm3.65}$              \\
\ \ with RFR      & 50.56$_{\pm1.47}$               & 48.48$_{\pm1.17}$              & 44.99$_{\pm0.97}$              & 57.17$_{\pm1.66}$               & 55.62$_{\pm0.71}$              & 54.40$_{\pm1.63}$              \\ \cmidrule(l){2-7} 
Improvement         & +0.72                            & +1.62                           & +0.3                            & +3.4                             & +7.21                           & +4.61                           \\ \midrule
LwF~\citep{li2017learning}           & 39.25$_{\pm2.04}$               & 26.53$_{\pm1.87}$              & 33.41$_{\pm2.06}$              & 54.26$_{\pm1.32}$               & 50.30$_{\pm1.96}$              & 44.52$_{\pm1.70}$              \\
\ \ with RFR      & 40.98$_{\pm0.34}$               & 28.82$_{\pm4.17}$              & 34.46$_{\pm0.63}$              & 56.48$_{\pm1.22}$               & 52.52$_{\pm0.35}$              & 47.57$_{\pm1.02}$              \\ \cmidrule(l){2-7} 
Improvement         & +1.73                            & +2.29                           & +1.06                           & +2.22                            & +2.21                           & +3.05                           \\ \midrule
SI~\citep{zenke2017continual} & 37.90$_{\pm2.39}$               & 22.88$_{\pm1.41}$              & 31.01$_{\pm3.24}$              & 51.89$_{\pm1.70}$               & 49.25$_{\pm2.19}$              & 44.16$_{\pm1.64}$              \\
\ \ with RFR      & 39.41$_{\pm0.60}$               & 24.02$_{\pm0.38}$              & 33.48$_{\pm0.80}$              & 54.99$_{\pm1.08}$               & 51.98$_{\pm0.89}$              & 47.46$_{\pm1.09}$              \\ \cmidrule(l){2-7} 
Improvement         & +1.51                            & +1.13                           & +2.47                           & +3.1                             & +2.73                           & +3.29                           \\ \midrule
EEIL~\citep{castro2018end}          & 37.81$_{\pm2.55}$               & 24.48$_{\pm1.18}$              & 32.37$_{\pm2.52}$              & 52.04$_{\pm1.90}$               & 48.65$_{\pm2.34}$              & 43.95$_{\pm1.65}$              \\
\ \ with RFR      & 39.52$_{\pm0.38}$               & 25.79$_{\pm1.18}$              & 33.46$_{\pm0.69}$              & 54.74$_{\pm0.91}$               & 51.64$_{\pm0.80}$              & 47.89$_{\pm0.60}$              \\ \cmidrule(l){2-7} 
Improvement         & +1.72                            & +1.31                           & +1.09                           & +2.71                            & +2.98                           & +3.94                           \\ \midrule
MAS~\citep{aljundi2018memory}           & 38.02$_{\pm2.49}$               & 26.45$_{\pm1.18}$              & 31.66$_{\pm3.20}$              & 52.17$_{\pm1.75}$               & 50.01$_{\pm2.27}$              & 46.61$_{\pm1.15}$              \\
\ \ with RFR      & 39.54$_{\pm0.71}$               & 28.80$_{\pm4.52}$              & 33.36$_{\pm0.52}$              & 55.04$_{\pm1.30}$               & 52.77$_{\pm1.02}$              & 50.32$_{\pm1.38}$              \\ \cmidrule(l){2-7} 
Improvement         & +1.53                            & +2.35                           & +1.71                           & +2.88                            & +2.76                           & +3.71                           \\ \midrule
RWalk~\citep{chaudhry2018riemannian}         & 35.75$_{\pm1.63}$               & 23.34$_{\pm2.06}$              & 27.02$_{\pm5.36}$              & 41.08$_{\pm1.91}$               & 21.44$_{\pm5.69}$              & 20.81$_{\pm2.12}$              \\
\ \ with RFR      & 39.00$_{\pm1.15}$               & 24.44$_{\pm1.06}$              & 32.85$_{\pm1.40}$              & 44.83$_{\pm1.87}$               & 35.42$_{\pm1.71}$              & 22.68$_{\pm4.70}$              \\ \cmidrule(l){2-7} 
Improvement         & +3.25                            & +1.1                            & +5.82                           & +3.75                            & +13.98                          & +1.88                           \\ \midrule
BiC~\citep{wu2019large}           & 54.36$_{\pm1.23}$               & 43.42$_{\pm2.70}$              & 26.73$_{\pm1.89}$              & 61.00$_{\pm4.69}$               & 53.12$_{\pm5.68}$              & 25.44$_{\pm1.74}$              \\
\ \ with RFR      & 58.06$_{\pm0.73}$               & 48.02$_{\pm2.33}$              & 27.74$_{\pm3.11}$              & 64.18$_{\pm1.94}$               & 53.77$_{\pm3.05}$              & 27.56$_{\pm2.72}$              \\ \cmidrule(l){2-7} 
Improvement         & +3.71                            & +4.60                           & +1.01                           & +3.18                            & +0.65                           & +2.13                           \\ \midrule

IL2M~\citep{belouadah2019il2m}          & 40.85$_{\pm2.22}$               & 24.92$_{\pm1.12}$              & 33.41$_{\pm2.12}$              & 56.78$_{\pm1.34}$               & 52.35$_{\pm2.16}$              & 47.10$_{\pm1.85}$              \\
\ \ with RFR      & 42.54$_{\pm0.49}$               & 25.19$_{\pm0.41}$              & 35.20$_{\pm0.78}$              & 58.40$_{\pm1.04}$               & 55.25$_{\pm0.78}$              & 50.09$_{\pm0.83}$              \\ \cmidrule(l){2-7} 
Improvement         & +1.69                            & +0.27                           & +1.78                           & +1.62                            & +2.9                            & +2.99                           \\ \midrule
UCIR~\citep{hou2019learning}          & 66.30$_{\pm0.36}$               & 60.57$_{\pm0.56}$              & 52.74$_{\pm0.72}$              & 70.57$_{\pm0.51}$               & 67.62$_{\pm0.37}$              & 63.22$_{\pm0.42}$              \\
\ \ with RFR      & 69.45$_{\pm0.29}$               & 66.16$_{\pm0.10}$              & 61.23$_{\pm0.13}$              & 71.65$_{\pm0.52}$               & 69.52$_{\pm0.13}$              & 65.46$_{\pm0.57}$              \\ \cmidrule(l){2-7} 
Improvement         & +3.14                            & +5.59                           & +8.49                           & +1.08                            & +1.90                           & +2.24                           \\ \midrule

AFC~\citep{kang2022class}          & 70.35$_{\pm0.29}$               & 67.02$_{\pm0.24}$              & 65.05$_{\pm0.27}$              & 77.10$_{\pm0.03}$               & 75.27$_{\pm0.18}$              & 72.90$_{\pm0.44}$              \\
\ \ with RFR      & 70.85$_{\pm0.21}$               & 67.82$_{\pm0.37}$              & 65.53$_{\pm0.24}$              & 77.19$_{\pm0.08}$               & 75.47$_{\pm0.23}$              & 72.99$_{\pm0.04}$              \\ \cmidrule(l){2-7} 
Improvement         & +0.50                            & +0.80                           & +0.48                           & +0.09                            & +0.20                           & +0.09                           \\ \midrule

Average Improvement & +1.95   & +2.11 & +2.42 & +2.40 & +3.75 & +2.79 \\ \bottomrule
\end{tabular}
\end{table*}

\begin{figure*}[ht!]
    \centering
    \begin{subfigure}{0.30\linewidth}
        \includegraphics[width=\linewidth]{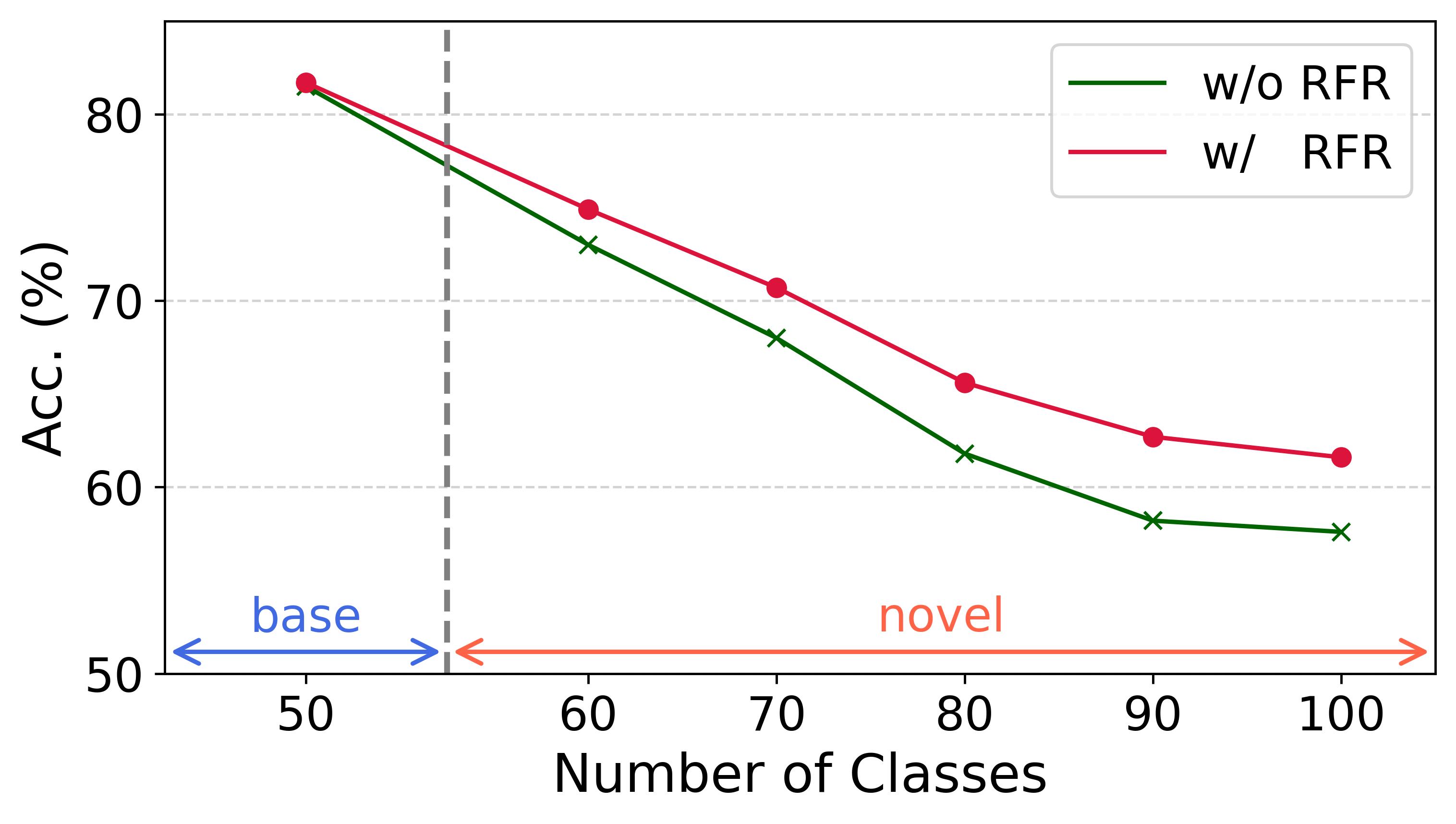}
        \caption{S=10}
    \end{subfigure}
    \begin{subfigure}{0.30\linewidth}
        \includegraphics[width=\linewidth]{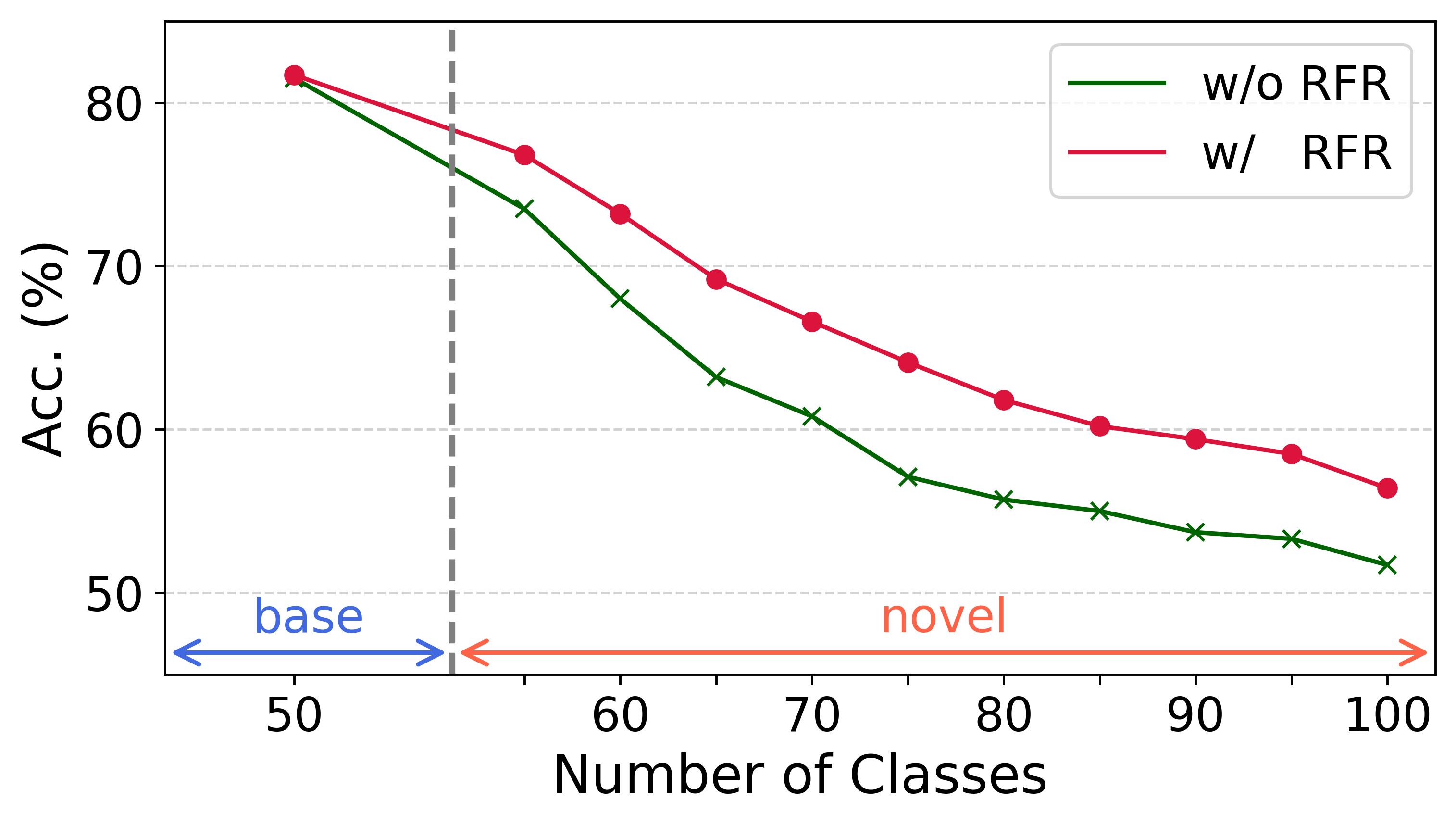}
        \caption{S=5}
    \end{subfigure}
    \begin{subfigure}{0.30\linewidth}
        \includegraphics[width=\linewidth]{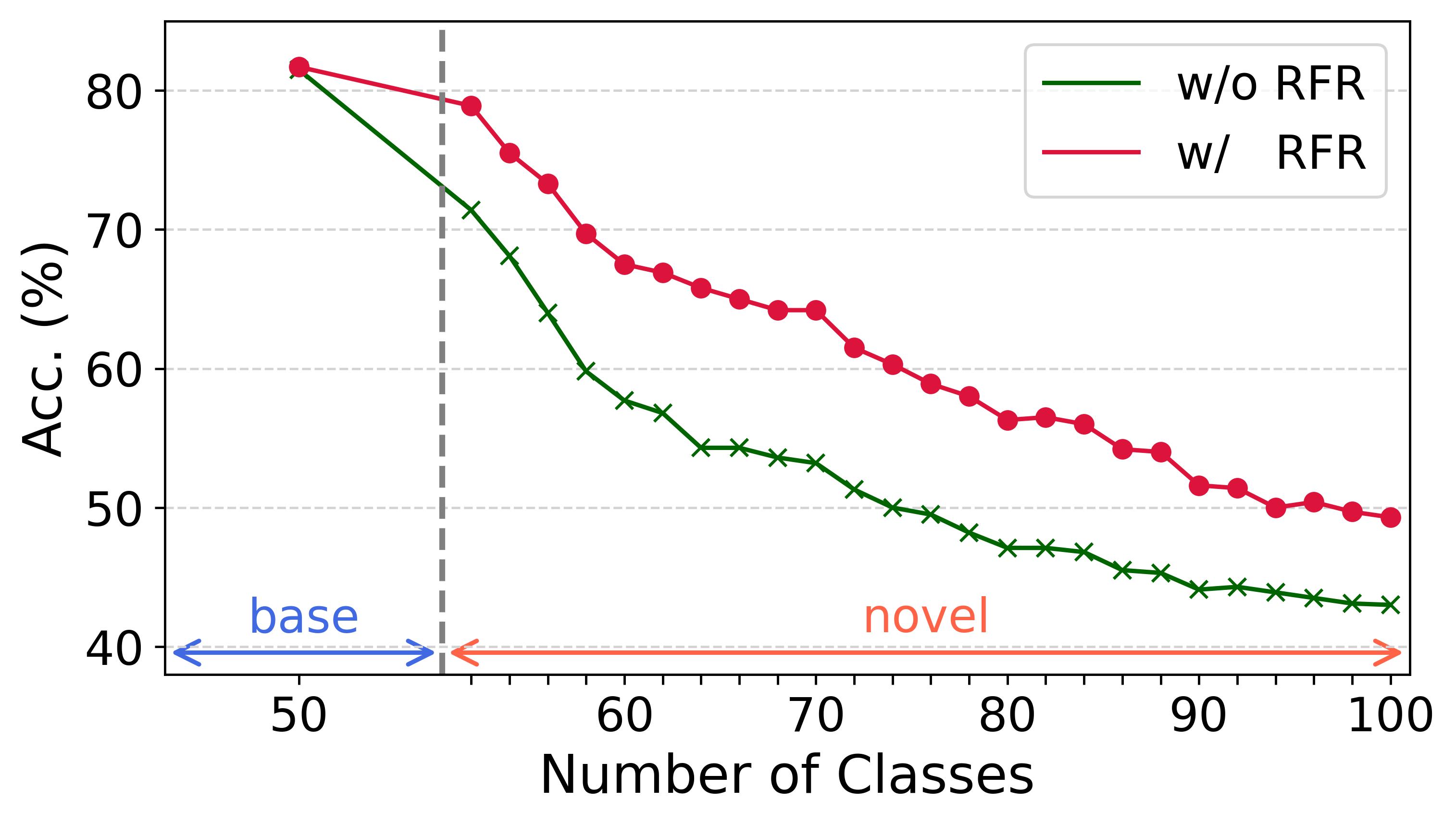}
        \caption{S=2}
    \end{subfigure}
    \caption{Overall accuracy at each session for UCIR. 
    Two ResNet-18 models are trained with and without RFR for \textit{CIFAR-100 dataset}, utilizing 50 base classes under different split sizes (a) 10, (b) 5, and (c) 2 for each novel session.
    }
    \label{fig:acc_each_phase_cifar}
\end{figure*}

\begin{table*}[ht!]
\caption{\label{tab:sota_experiment2} 
Performance improvements by RFR for PODNet. Unlike the other works, PODNet used its own class orderings for evaluation. As an effort to make comparisons as fair as possible, we followed PODNet's class orderings for this table.
}
\centering
\begin{tabular}{@{}lrrrrrr@{}}
\toprule
Method      & \multicolumn{3}{c}{CIFAR100 (B=50)}                                                               & \multicolumn{3}{c}{ImageNet-100 (B=50)}                                                           \\ \cmidrule(lr){2-4} \cmidrule(l){5-7}
            & S=10       & S=5        & S=2        & S=10       & S=5        & S=2        \\ \midrule
PODNet(NME)~\citep{douillard2020podnet} & 68.47$_{\pm1.27}$               & 67.09$_{\pm1.19}$              & 65.16$_{\pm1.02}$              & 60.42$_{\pm0.53}$               & 50.93$_{\pm0.73}$              & 36.36$_{\pm0.30}$              \\
\ \ with RFR    & 69.51$_{\pm0.56}$               & 67.66$_{\pm0.63}$              & 65.59$_{\pm0.48}$              & 64.08$_{\pm0.45}$               & 56.34$_{\pm0.75}$              & 41.03$_{\pm2.15}$              \\ \cmidrule(l){2-7} 
Improvement       & +1.04                            & +0.57                           & +0.43                           & +3.66                            & +5.42                           & +4.67                           \\ \midrule
PODNet(CNN) & 66.11$_{\pm0.63}$               & 63.49$_{\pm0.45}$              & 59.61$_{\pm0.53}$              & 73.76$_{\pm0.17}$               & 68.32$_{\pm0.27}$              & 61.82$_{\pm0.68}$              \\
\ \ with RFR    & 67.13$_{\pm0.86}$               & 64.72$_{\pm1.01}$              & 61.49$_{\pm0.71}$              & 74.42$_{\pm0.27}$               & 69.66$_{\pm0.09}$              & 63.11$_{\pm0.22}$              \\ \cmidrule(l){2-7} 
Improvement       & +1.02                            & +1.23                           & +1.89                           & +0.66                            & +1.34                           & +1.29                          \\ \midrule
Average Improvement & +1.03   & +0.90 & +1.16 & +2.16 & +3.38 & +2.98 \\ \bottomrule
\end{tabular}
\end{table*}

To demonstrate the efficacy of RFR in enhancing performance, we conduct comprehensive experiments over eleven well-known existing works, all utilizing a single backbone. Notably, dynamic network methods~\citep{yan2021dynamically,douillard2022dytox,wang2022foster,wang2022beef} were excluded, as they generally use separate backbones for each task. This design inherently makes them more resistant to catastrophic forgetting~\citep{zhou2024class}, thereby reducing their reliance on forward compatibility. However, this advantage comes at the cost of increased computational demands with each new session.

The results presented in Table~\ref{tab:sota_experiment1} and Table~\ref{tab:sota_experiment2} provide compelling evidence of the significant and consistent performance improvements achieved through the integration of RFR into the previous works. Specifically, we observe notable average performance improvements for the state-of-the-art works, including a 3.74\% average improvement for UCIR, a 2.63\% average improvement for PODNet (NME), a 1.24\% average improvement for PODNet (CNN), and a 0.36\% average improvement for AFC.
Particularly, a remarkable performance improvement is observed in the case of UCIR for CIFAR-100 and its further analysis is shown in Figure~\ref{fig:acc_each_phase_cifar}. The achieved improvements are 3.14\% increase for S=10, 5.59\% for S=5, and 8.49\% for S=2.

\begin{table}[h!]
\caption{\label{tab:exemplar_free} 
Performance improvements by RFR for a non-exemplar based Class-IL method, IL2A~\citep{zhu2021class}. All results are trained using CIFAR-100 with its own class orderings.
}
\centering
\begin{tabular}{@{}lrrr@{}}
\toprule
Method          & S=10       & S=5        & S=2        \\ \midrule
IL2A            & 65.53$_{\pm0.46}$ & 58.58$_{\pm0.76}$ & 54.31$_{\pm0.54}$ \\
\ \ with RFR    & 66.07$_{\pm0.33}$ & 60.06$_{\pm0.48}$ & 55.76$_{\pm0.12}$ \\ \cmidrule(l){2-4} 
Improvement     & +0.54      & +1.48      & +1.45      \\ \bottomrule
\end{tabular}
\end{table}

\begin{table}[h!]
\caption{\label{tab:imagenet1k} 
Performance improvements by RFR with larger-scale dataset, ImageNet-1000.
}
\centering
\begin{tabular}{@{}lrr@{}}
\toprule
Method       & \multicolumn{2}{c}{ImageNet-1000 (B=500)}   \\ \cmidrule(l){2-3} 
             & S=100                 & S=50                \\ \midrule
UCIR         & 65.38$_{\pm0.04}$     & 63.44$_{\pm0.11}$   \\
\ \ with RFR & 65.69$_{\pm0.10}$     & 63.64$_{\pm0.25}$   \\ \cmidrule(l){2-3} 
Improvement  & +0.31                 & +0.20               \\ \bottomrule
\end{tabular}
\end{table}

\begin{table}[h!]
\caption{\label{tab:resnet50} 
Performance improvements by RFR with larger-scale models, ResNet-50 and ViT-Small.
}
\centering
\resizebox{\linewidth}{!}{
\begin{tabular}{@{}llrrr@{}}
\toprule
Model       & Method       & S=10               & S=5               & S=2               \\ \midrule
ResNet-50   & UCIR      & 67.10$_{\pm0.20}$  & 61.42$_{\pm0.61}$ & 51.62$_{\pm0.14}$ \\
            & \ \ with RFR      & 68.57$_{\pm0.12}$  & 64.81$_{\pm0.15}$ & 57.71$_{\pm0.31}$ \\ \cmidrule(l){3-5} 
& Improvement & +1.47              & +3.39             & +6.09             \\ \midrule
ViT-Small   & UCIR      & 50.45$_{\pm0.35}$   & 49.03$_{\pm0.27}$  & 46.24$_{\pm0.79}$ \\
            & \ \ with RFR      & 51.86$_{\pm0.35}$   & 50.50$_{\pm0.47}$  & 47.96$_{\pm1.03}$ \\ \cmidrule(l){3-5} 
& Improvement & +1.41               & +1.47              & +1.72             \\ \bottomrule
\end{tabular}
}
\end{table}

\subsection{Further analysis}
\label{subsec:further_analysis}

\paragraph{Non-exemplar-based approach:}
While the effectiveness of RFR is demonstrated for the most popular CIL methods in Section~\ref{subsec:sota_experiment}, they all belong to exemplar-based approaches. Therefore, we further conduct an analysis of RFR for the non-exemplar-based CIL approach, IL2A~\citep{zhu2021class}. The results presented in Table~\ref{tab:exemplar_free} exhibit significant performance enhancements, indicating an average improvement of 1.16\%.

\paragraph{Larger-scale dataset and models:}
In addition to evaluating RFR's performance under the widely used benchmark settings, we extend our evaluation to a larger-scale dataset. 
Due to the high computational demands of the ImageNet-1000 dataset and our limited resources, we evaluate RFR using only a single method.
As shown in Table~\ref{tab:imagenet1k}, RFR continues to demonstrate its effectiveness for ImageNet-1000 dataset.
Furthermore, we analyze RFR's performance with larger-scale models. The results presented in Table~\ref{tab:resnet50} underscore RFR's efficacy across a range of larger-scale models, including ResNet-50 and ViT-Small~\citep{lee2021vision}.

\subsection{Ablation study}
\label{subsec:ablation_study}

\paragraph{Sensitivity study of regularization coefficient:}
We performed a sensitivity study of the strength hyper-parameter $\alpha$ in Eq.~(\ref{eq:loss}).
The results are presented in Figure~\ref{fig:ablation_coef}, and they demonstrate consistent and smooth inverted U-shaped patterns in performance, with the peak performance observed at the value of $\alpha=0.1$.

\begin{figure*}[ht!]
    \centering
    \begin{subfigure}{0.30\linewidth}
        \includegraphics[width=\linewidth]{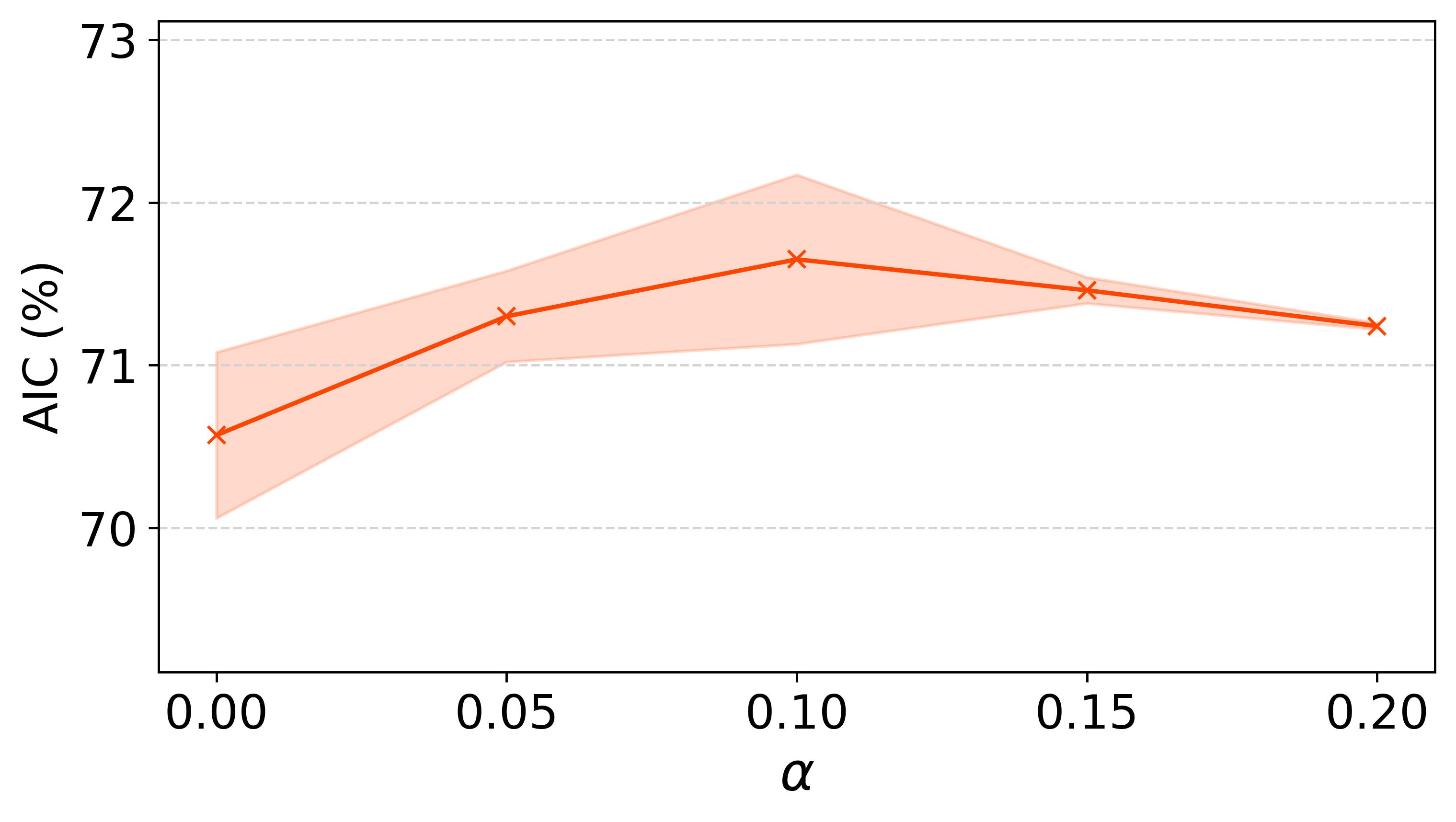}
        \caption{S=10}
    \end{subfigure}
    \begin{subfigure}{0.30\linewidth}
        \includegraphics[width=\linewidth]{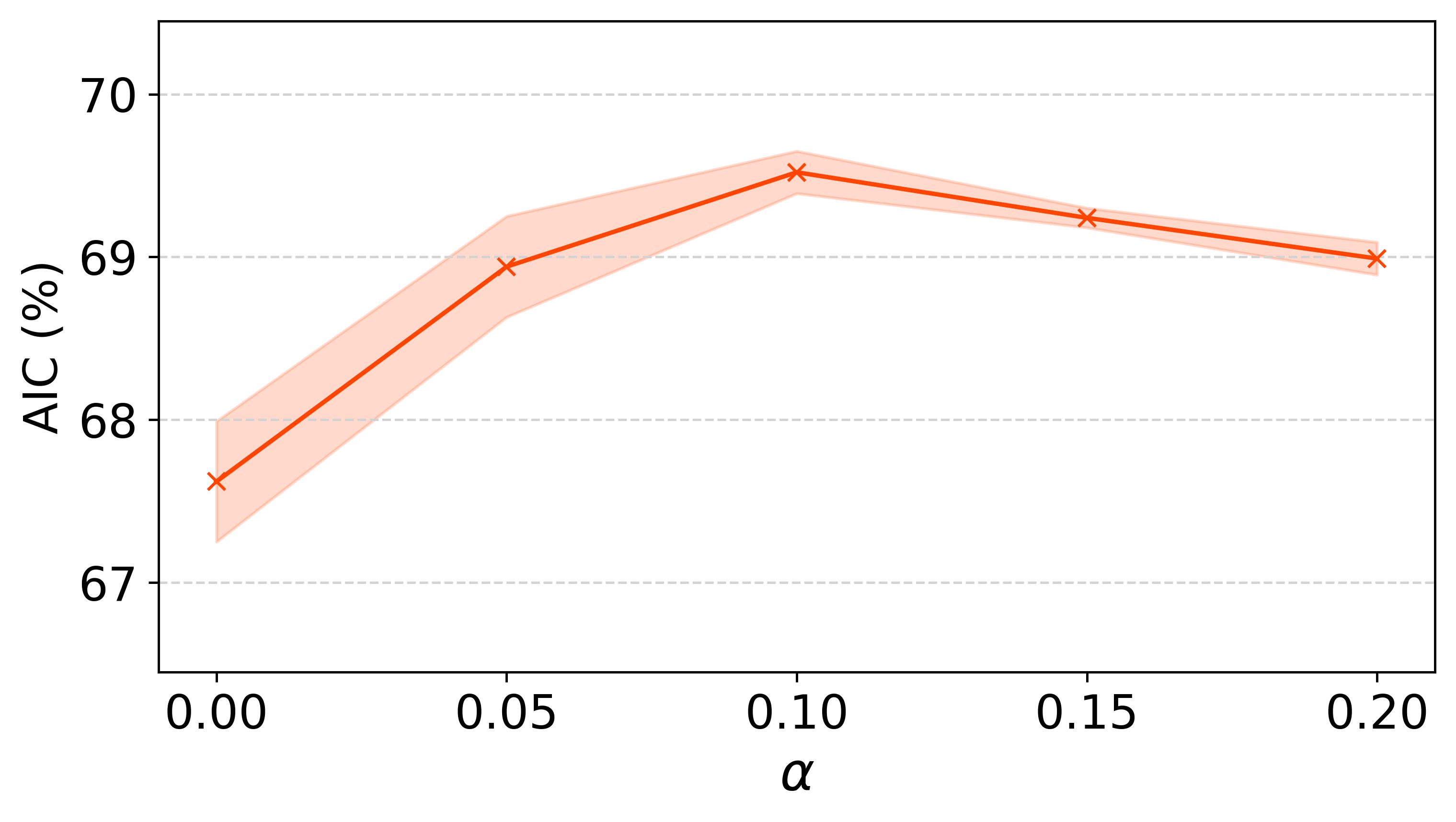}
        \caption{S=5}
    \end{subfigure}
    \begin{subfigure}{0.30\linewidth}
        \includegraphics[width=\linewidth]{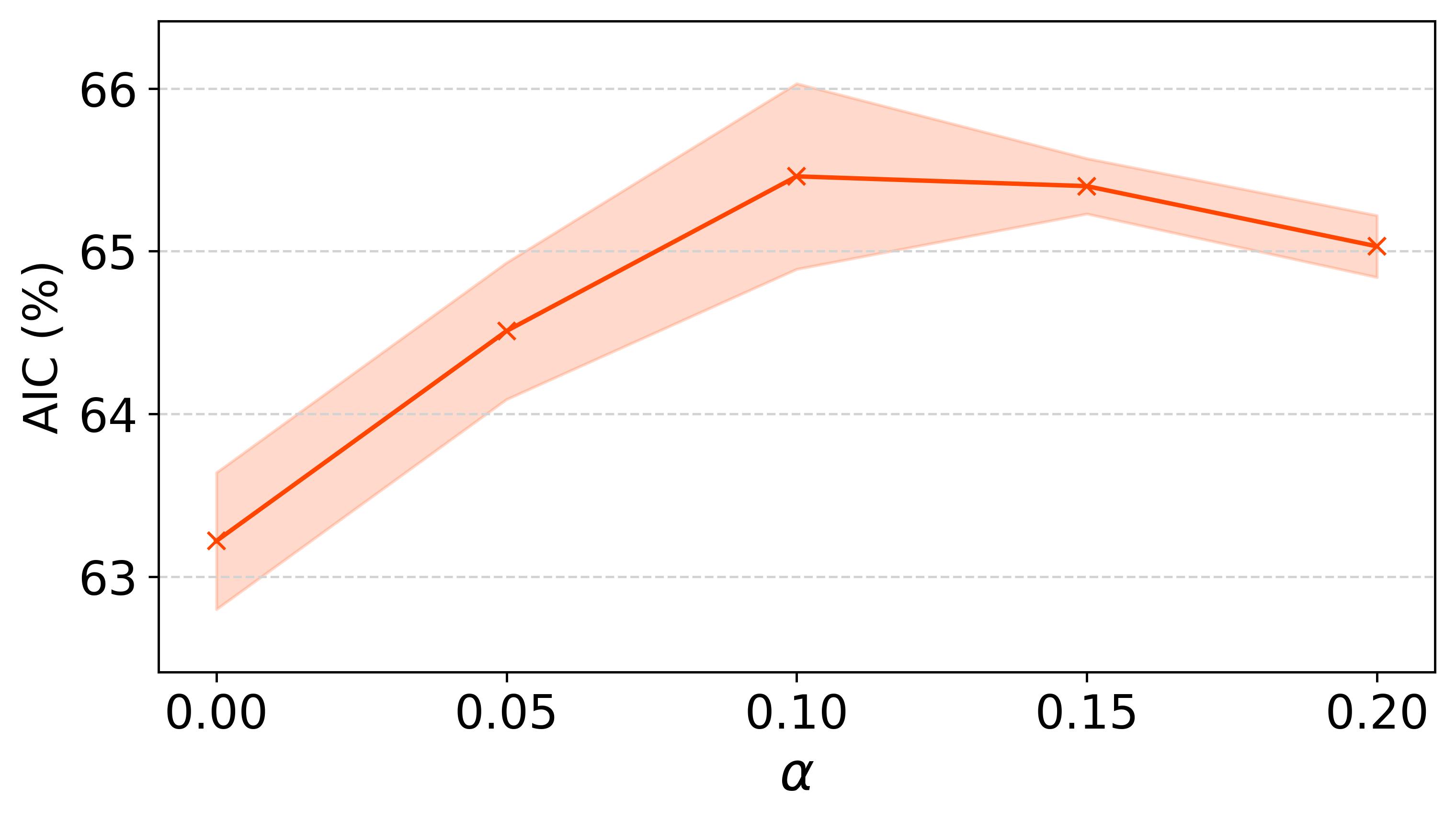}
        \caption{S=2}
    \end{subfigure}
    
    \caption{Impact of regularization coefficient on the performance.
    ResNet-18 models are trained with a range of regularization coefficients for ImageNet-100 dataset, utilizing 50 base classes under different split sizes (a) 10, (b) 5, and (c) 2 for each novel session.
    }
    \label{fig:ablation_coef}
\end{figure*}

\paragraph{Impact of the number of classes in the base task:}
We demonstrate the consistent effectiveness of RFR across various numbers of classes in the base task. The results are presented in Table~\ref{tab:various_base_tasks}. Specifically, RFR's effectiveness is relatively lower when B=10, potentially attributable to the limited features within the 10 classes. However, as the number of classes in the base task increases, the effectiveness of RFR also increases, likely due to the expanded availability of features within the larger number of classes.

\begin{table}[h!]
\caption{\label{tab:various_base_tasks}
Performance improvements by RFR for various number of classes in the base task. All results are trained using CIFAR-100 with the same class orderings as originally proposed in iCaRL~\citep{rebuffi2017icarl}.
}
\centering
\begin{tabular}{@{}cccc@{}}
\toprule
Classes in base tasks      & UCIR        & with RFR         & Improvement \\ \midrule
B=10 & 60.31$_{\pm0.22}$ & 60.41$_{\pm0.34}$ & +0.10        \\
B=20 & 63.16$_{\pm0.34}$ & 65.11$_{\pm0.29}$ & +1.95        \\
B=30 & 65.17$_{\pm0.25}$ & 66.98$_{\pm0.21}$ & +1.81        \\
B=40 & 67.22$_{\pm0.25}$ & 69.53$_{\pm0.12}$ & +2.31        \\
B=50 & 66.30$_{\pm0.36}$ & 69.45$_{\pm0.29}$ & +3.15        \\
B=60 & 70.44$_{\pm0.13}$ & 73.06$_{\pm0.36}$ & +2.62        \\
B=70 & 71.57$_{\pm0.39}$ & 74.20$_{\pm0.20}$ & +2.63        \\
B=80 & 73.07$_{\pm0.18}$ & 75.95$_{\pm0.09}$ & +2.88        \\
B=90 & 73.41$_{\pm0.43}$ & 75.98$_{\pm0.19}$ & +2.57        \\ \bottomrule
\end{tabular}
\end{table}

\paragraph{Time cost analysis:}
We conducted a time cost analysis of Base Session Training using various datasets. As demonstrated in Table~\ref{tab:cost_analysis}, the experimental results indicate a slight increase in the average time; however, considering the standard deviation, this difference appears to be negligible.
This result is likely due to the highly efficient implementation of Singular Value Decomposition (SVD) in the CUBLAS library~\citep{lahabar2009singular}, which is the default linear algebra library used by PyTorch.

\begin{table}[h!]
\caption{\label{tab:cost_analysis}
Time Cost Analysis of Base Session Training with Various Datasets.
The average training time (in seconds) per epoch was measured for each dataset.
Each model was trained for three runs using an NVIDIA RTX 3090 GPU.
}
\centering
\begin{tabular}{@{}lccc@{}}
\toprule
         & CIFAR100   & ImageNet-100 & ImageNet-1000 \\ \cmidrule(l){2-4} 
         & B=50       & B=50         & B=500         \\ \midrule
UCIR     & 20.22$_{\pm0.18}$ & 145.00$_{\pm0.76}$  & 995.53$_{\pm15.10}$  \\
\ \ with RFR & 20.50$_{\pm0.21}$ & 146.48$_{\pm3.20}$  & 996.96$_{\pm15.30}$  \\ \bottomrule
\end{tabular}
\end{table}

\section{Discussion}
\label{discussion}

\subsection{Comparison with CwD}
\label{subsec:cwd_comparison}

Recently, CwD~\citep{shi2022mimicking}, another forward-compatible method, has improved previous state-of-the-art methods in CIL by enforcing class-wise decorrelation. Because both RFR and CwD regularize representation during the base session, a comparative analysis of the two methods was conducted.
Specifically, CwD regularizes the Frobenius norm of representation within the same class, while RFR regularizes the eigenvalues of representation without considering class information. However, it has been demonstrated that regularizing the eigenvalues of representation is more effective than regularizing the Frobenius norm in~\citep{kim2023vne}. The effectiveness of RFR is further confirmed through subsequent evaluation.
The performance evaluation was carried out using the UCIR and PODNet models with the CIFAR-100 dataset. As shown in Table~\ref{tab:comparison_cwd}, RFR consistently demonstrates superior performance, with significant improvements of 1.76\%, 2.27\%, and 3.10\% on average for split sizes of 10, 5, and 2, respectively. These improvements tend to be larger when dealing with smaller split sizes, or equivalently, when handling a larger number of novel sessions.

\begin{table*}[ht!]
\caption{\label{tab:comparison_cwd} Comparison with CwD.
}
\centering
\begin{tabular}{@{}lcccccc@{}}
\toprule
         & \multicolumn{6}{c}{CIFAR100 (B=50)}                                                                                                                                         \\ \cmidrule(l){2-7} 
         & \multicolumn{3}{c}{UCIR}                                                             & \multicolumn{3}{c}{PODNet (CNN)}                                                     \\ \cmidrule(l){2-4} \cmidrule(l){5-7}
         & S=10                & S=5                 & S=2                 & S=10                & S=5                 & S=2                 \\ \midrule
Baseline & 66.30$_{\pm0.36}$          & 60.57$_{\pm0.56}$          & 52.74$_{\pm0.72}$          & 66.11$_{\pm0.63}$          & 63.49$_{\pm0.45}$          & 59.61$_{\pm0.53}$          \\ \midrule
\ \ with CwD      & 67.06$_{\pm0.12}$          & 62.71$_{\pm0.36}$          & 56.39$_{\pm0.30}$          & 66.01$_{\pm0.75}$          & 63.63$_{\pm1.02}$          & 60.14$_{\pm1.16}$          \\
\ \ with RFR      & 69.45$_{\pm0.29}$ & 66.16$_{\pm0.10}$ & 61.23$_{\pm0.13}$ & 67.13$_{\pm0.86}$ & 64.72$_{\pm1.01}$ & 61.49$_{\pm0.71}$ \\ \midrule
Difference & +2.39		& +3.45		& +4.84		& +1.12		& +1.09		& +1.35	
\\ \bottomrule
\end{tabular}
\end{table*}

\subsection{Comparisons to the most recent methods}
\begin{table*}[h!]
\caption{\label{tab:latest_comparison} 
Comparisons to the most recent CIL methods.
}
\centering
\begin{tabular}{@{}lcccc@{}}
\toprule
Method       & \multicolumn{2}{c}{CIFAR100 (B=50)}            & \multicolumn{2}{c}{ImageNet-100 (B=50)}        \\ \cmidrule(l){2-3} \cmidrule(l){4-5} 
             & S=10              & S=5                        & S=10              & S=5                        \\ \midrule
DCMI~\citep{qiu2024dual}         & 67.90             & 66.80                      & 70.50             & 70.00                      \\
FCS~\citep{li2024fcs}          & 62.13             & 60.39                      & N/A               & 61.76                      \\
MRFA~\citep{zhengmulti}         & \textbf{74.73}    & N/A                        & \textbf{79.29}    & N/A                        \\ \midrule
AFC~\citep{kang2022class} with RFR & 70.85$_{\pm0.21}$ & \textbf{67.82$_{\pm0.37}$} & 77.19$_{\pm0.08}$ & \textbf{75.47$_{\pm0.23}$} \\ \bottomrule
\end{tabular}
\end{table*}
While the benchmark methods presented in Section~\ref{subsec:sota_experiment} are widely adopted in CIL due to their popularity and reliably reproducible performance, they do not include the most recent methods in CIL. In this section, we compare the performance of RFR with the most recent methods~\citep{qiu2024dual,li2024fcs,zhengmulti}. Since we could not reproduce the originally reported performance of these methods, we present their reported performance from their respective papers and compare it with our performance (AFC with RFR). The results are shown in Table~\ref{tab:latest_comparison}. Although MRFA~\citep{zhengmulti} demonstrates the best performance when S=10, our method achieves the second-best performance for S=10. For S=5, our method surpasses the other recent methods. These results indicate that the benchmark performance comparisons in Section~\ref{subsec:sota_experiment} remain competitive also for the recent methods.

\subsection{Increasing representation rank during novel sessions}
\label{subsec:discuss_novel_sessions}

RFR increases effective rank during the base session only. 
It is also possible to increase effective rank during novel sessions with the goal of acquiring additional features that may be useful in the subsequent novel sessions.
Such a strategy, however, can also intensify catastrophic forgetting during novel sessions. When additional regularization, such as increasing the effective rank, is applied during the training of a novel task, it results in more significant changes to the model compared to when only the novel task is learned. This occurs because the model must simultaneously learn both the novel task and the regularization task. Consequently, the regularization intended to enhance the model's forward compatibility with subsequent tasks introduces a trade-off, increasing catastrophic forgetting of previous tasks.
To investigate the overall effect, we have conducted an experiment and the results are shown in Table~\ref{tab:comparison_novel_reg}.
While RFR demonstrates effectiveness in enhancing feature richness during the base session, its influence on overall performance is marginal when also applied to novel sessions. Therefore, we have chosen to apply effective rank regularization only during the base session.

\begin{table*}[ht!]
\caption{\label{tab:comparison_novel_reg} Influence of increasing representation rank during novel sessions.
}
\centering
\resizebox{\linewidth}{!}{
\begin{tabular}{@{}lccccccccc@{}}
\toprule
           & \multicolumn{9}{c}{CIFAR100 (B=50)}                                                                                                           \\ \cmidrule(l){2-10} 
           & \multicolumn{3}{c}{S=10}               & \multicolumn{3}{c}{S=5}                & \multicolumn{3}{c}{S=2}                \\ \cmidrule(l){2-4} \cmidrule(l){5-7} \cmidrule(l){8-10}
           & Base Only   & Base+Novel       & Diff. & Base Only   & Base+Novel       & Diff. & Base Only   & Base+Novel       & Diff. \\ \midrule
BiC        & 58.06$_{\pm0.73}$ & 58.29$_{\pm0.38}$ & +0.23  & 48.02$_{\pm2.33}$ & 46.60$_{\pm1.91}$ & -1.41 & 27.74$_{\pm3.11}$ & 27.64$_{\pm1.83}$ & -0.10 \\
EEIL       & 39.52$_{\pm0.38}$ & 39.44$_{\pm0.63}$ & -0.08 & 25.79$_{\pm1.18}$ & 25.62$_{\pm1.10}$ & -0.17 & 33.46$_{\pm0.69}$ & 33.16$_{\pm1.07}$ & -0.29 \\
iCaRL      & 50.56$_{\pm1.47}$ & 50.50$_{\pm2.04}$ & -0.06 & 48.48$_{\pm1.17}$ & 48.23$_{\pm1.14}$ & -0.25 & 44.99$_{\pm0.97}$ & 44.81$_{\pm1.18}$ & -0.18 \\
IL2M       & 42.54$_{\pm0.49}$ & 42.76$_{\pm0.83}$ & +0.21  & 25.19$_{\pm0.41}$ & 26.22$_{\pm1.77}$ & +1.03  & 35.20$_{\pm0.78}$ & 34.87$_{\pm0.67}$ & -0.32 \\
LwF        & 40.98$_{\pm0.34}$ & 40.69$_{\pm0.81}$ & -0.28 & 28.82$_{\pm4.17}$ & 26.70$_{\pm0.39}$ & -2.12 & 34.46$_{\pm0.63}$ & 34.63$_{\pm0.75}$ & +0.17  \\
MAS        & 39.54$_{\pm0.71}$ & 39.77$_{\pm0.66}$ & +0.23  & 28.80$_{\pm4.52}$ & 29.03$_{\pm3.50}$ & +0.23  & 33.36$_{\pm0.52}$ & 33.91$_{\pm0.91}$ & +0.55  \\
SI         & 39.41$_{\pm0.60}$ & 39.51$_{\pm0.58}$ & +0.10  & 24.02$_{\pm0.38}$ & 23.89$_{\pm0.19}$ & -0.13 & 33.48$_{\pm0.80}$ & 33.60$_{\pm0.52}$ & +0.12  \\
RWalk      & 39.00$_{\pm1.15}$ & 38.91$_{\pm1.39}$ & -0.09 & 24.44$_{\pm1.06}$ & 25.96$_{\pm3.06}$ & +1.53  & 32.85$_{\pm1.40}$ & 33.04$_{\pm1.20}$ & +0.19  \\
UCIR       & 69.45$_{\pm0.29}$ & 69.57$_{\pm0.23}$ & +0.12  & 66.16$_{\pm0.10}$ & 66.55$_{\pm0.16}$ & +0.39  & 61.23$_{\pm0.13}$ & 61.27$_{\pm0.29}$ & +0.04  \\ \midrule
Average difference &                   &                   & +0.04  &                   &                   & -0.10 &                   &                   & +0.02  \\ \bottomrule
\end{tabular}
}
\end{table*}

\subsection{Practical application perspective}
In this study, we demonstrated that RFR is effective across models of various sizes and architectures, including both convolution-based architectures (e.g., ResNet-18 and ResNet-50) and transformer-based architectures (e.g., ViT-Small). Given that RFR introduces minimal additional training time and is effective in single-network methods, which are relatively resource-efficient~\citep{zhou2024class}, it suggests that RFR may offer greater advantages in environments with limited training resources, such as edge devices.

\subsection{Limitations}
The effectiveness of RFR depends on the accurate estimation of the effective rank of the representation. Since the effective rank is mathematically upper-bounded by the batch size, RFR regularization tends to be less efficacious with smaller batch sizes. Empirical evidence indicates that batch sizes greater than 32 are sufficient to achieve consistent and reliable performance.

We demonstrate the efficacy of RFR in enhancing forward compatibility by applying it to CIL methods that use a single backbone. The limited effectiveness of RFR when applied to dynamic network methods with increasing backbones underscores a limitation of this approach. Additionally, due to the high computational demands of the ImageNet-1000 dataset and limited resources, we evaluated RFR using only a single method.

We have followed previous works and presented the main results in Table~\ref{tab:sota_experiment1} and Table~\ref{tab:sota_experiment2}. Additionally, Table~\ref{tab:exemplar_free} focuses on non-exemplar based CIL, Table~\ref{tab:imagenet1k} addresses large-scale dataset, and Table~\ref{tab:resnet50} covers large-scale models. To ensure the generalizability of the proposed method, it would also be beneficial to consider real-world scenarios across a variety of application domains.

\section{Conclusion}

In this study, we propose effective-Rank based Feature Richness~(RFR) enhancement method that can be integrated with a wide range of existing methods in CIL.
More specifically, our method increases the effective rank of representations during the base session. In order to substantiate the effectiveness of our method, we have established a theoretical connection between the effective rank and the Shannon entropy of representations. Through empirical analysis, we have demonstrated the efficacy of our method in several dimensions including enhancement in forward compatibility, mitigation of catastrophic forgetting, and improvement in performance.
In summary, our method effectively enhances both forward and backward compatibility simultaneously. The superiority of this method is supported by both theoretical and empirical evidences.

\section*{Declaration of competing interest}
\label{sec:interest}
The authors declare that they have no known competing financial interests or personal relationships that could have influenced the work reported in this paper.
\section*{Acknowledgement}
\label{sec:acknowledge}

This work was supported by a National Research Foundation of Korea (NRF) grant funded by the Korea government (MSIT) (No. NRF-2020R1A2C2007139) and in part by Institute of Information \& Communications Technology Planning \& Evaluation (IITP) grant funded by the Korea government(MSIT) [NO.2021-0-01343, Artificial Intelligence Graduate School Program (Seoul National University)] and Basic Science Research Program through the National Research Foundation of Korea(NRF) funded by the Ministry of Education(NRF-2022R1A6A1A03063039).











\bibliographystyle{cas-model2-names}

\bibliography{cas-refs}



\end{document}